\documentclass[twoside]{article}

\usepackage[accepted]{aistats2024}
\usepackage[round]{natbib}

\usepackage[utf8]{inputenc} 
\usepackage[T1]{fontenc}    
\usepackage{hyperref}       
\usepackage{url}            
\usepackage{booktabs}       
\usepackage{amsfonts}       
\usepackage{nicefrac}       
\usepackage{microtype}      
\usepackage{xcolor}         

\usepackage{graphicx}
\usepackage{subfigure}
\usepackage{amsmath}
\usepackage{amssymb}
\usepackage{mathtools}
\usepackage{amsthm}

\theoremstyle{plain}
\newtheorem{theorem}{Theorem}[section]

\newtheorem{lemma}[theorem]{Lemma}

\theoremstyle{definition}
\newtheorem{definition}[theorem]{Definition}
\newtheorem{assumption}[theorem]{Assumption}
\theoremstyle{remark}
\newtheorem{remark}[theorem]{Remark}

\usepackage{paralist}

\newcommand{\ve}{{\varepsilon}}
\newcommand{\grad}{{\nabla F(\theta)}}
\newcommand{\hess}{{\nabla^2 F(\theta)}}
\newcommand{\gradGaus}{{\nabla F_\sigma(\theta)}}

\usepackage{stfloats} 
\usepackage{mathrsfs} 

\begin{document}

%

%

\twocolumn[

\aistatstitle{Implicit Two-Tower Policies}
\aistatsauthor{Yunfan Zhao*$^{1}$ \And Qingkai Pan*$^{2}$ \And  Krzysztof Choromanski*$^{3,4}$ \AND Deepali Jain$^{3}$ \And Vikas Sindhwani$^{3}$}
\aistatsaddress{}
]


\begin{abstract}
We present a new class of structured reinforcement learning policy-architectures, {\it Implicit Two-Tower (ITT)} policies, where the actions are chosen based on the attention scores of their learnable latent representations with those of the input states. By explicitly disentangling action from state processing in the policy stack, we achieve two main goals: substantial computational gains and better performance. Our architectures are compatible with both discrete and continuous action spaces. By conducting tests on $15$ environments from $\mathrm{OpenAI}$ $\mathrm{Gym}$ and $\mathrm{DeepMind}$ $\mathrm{Control}$ $\mathrm{Suite}$, we show that ITT-architectures are particularly suited for blackbox/evolutionary optimization and the corresponding policy training algorithms outperform their vanilla unstructured implicit counterparts as well as commonly used explicit policies. We complement our analysis by showing how techniques such as hashing and lazy tower updates, critically relying on the two-tower structure of ITTs, can be applied to obtain additional computational gains. 
\end{abstract}

\section{INTRODUCTION \& RELATED WORK}
\label{sec:intro}

We consider the problem of training a policy $\pi_{\theta}:\mathcal{S} \rightarrow \mathcal{A}$, parameterized by learnable $\theta \in \mathbb{R}^{D}$ for a reinforcement learning (RL) agent (\cite{rlintro, rlintro2, introrl2,singi2023decision,zhou2022pac,he2023robust_iros,huang2020deep}). The policy is a potentially stochastic mapping from the state-space ($\mathcal{S}$) to the action-space ($\mathcal{A}$), either continuous or discrete. The objective is to maximize the expected total reward $R$ defined as a possibly discounted sum of the partial rewards $r_{i}(\mathbf{s}_{i},\mathbf{a}_{i},\mathbf{s}_{i+1})$ for the transition from $\mathbf{s}_{i} \in \mathcal{S}$ to $\mathbf{s}_{i+1} \in \mathcal{S}$ via $\mathbf{a}_{i} \in \mathcal{A}$. The transition function: $T:\mathcal{S} \times \mathcal{A} \rightarrow \mathcal{S}$ as well as the partial reward function: $r:\mathcal{S} \times \mathcal{A} \times \mathcal{S} \rightarrow \mathbb{R}$ (both potentially stochastic) are defined by the environment. Hence, expected total rewards are computed over random state transitions and action choices. We call the sequence $(\mathbf{s}_{0},\mathbf{a}_{0},\mathbf{s}_{1},\mathbf{a}_{1},...,\mathbf{s}_{T})$ of states visited by the agent intertwined with the actions proposed by $\pi_{\theta}$, the \textit{rollout} of an agent. 

The most common way of encoding policy mapping $\pi_{\theta}$ is via a neural network taking states as inputs and explicitly outputting as the activations of the last layer proposed actions or the distributions over actions to sample from (for stochastic policies). We call such policies \textit{explicit}. While explicit policies were successfully applied in several RL scenarios: learning directly from pixels, hierarchical learning, robot locomotion and more (\cite{ppo, es, worldmodels, toeplitz_pl, wenhao, hierarchical,huang2022applications}), recent evidence shows that the expressiveness of the policy-architecture can be improved if the explicit model is substituted by an \textit{implicit} one.

The implicit model (\cite{haarnoja, ibc, mbp}) operates by learning a function $E_{\theta}:\mathcal{S} \times \mathcal{A} \rightarrow \mathbb{R}$ taking as an input a state-action pair and outputting a scalar value that can be interpreted as an energy \cite{xie2016theory,xu2022energy}. The optimal action $\mathbf{a}^{*}(\mathbf{s})$ for a given state $\mathbf{s}$ is chosen as a solution to the following energy-minimization problem:
\begin{equation}
\label{eq:implicit}
\mathbf{a}^{*}(\mathbf{s}) = \pi_{\theta}(\mathbf{s}) =  \mathrm{argmin}_{\mathbf{a} \in \mathcal{A}} E_{\theta}(\mathbf{s},\mathbf{a}).    
\end{equation}
Implicit models were recently demonstrated to provide strong performance in the behavioral cloning (BC) setting (\cite{ibc}), outperforming their regular explicit counterparts (e.g. mean squared error and mixture density BC policies), also for high-dimensional action-spaces and image inputs.
Interestingly, robots with deployed implicit policies were shown to learn sophisticated behaviours on various manipulation tasks requiring very high precision (\cite{ibc}).

New results, showing that the implicit mappings from states to actions given by Eq. \ref{eq:implicit} are capable of modeling multi-valued and even discontinuous functions with arbitrary precision (see: Theorem 1 and 2 in \cite{ibc}, and \cite{bianchini2021generalization}) with continuous energy-functions $E$ modeled by regular neural networks, shed light on that phenomenon. Thus adding the $\mathrm{argmin}$-operator provides a gateway to extend universal approximation results of regular neural networks to larger classes of functions, enabling us to approximate with our neural network models classes of functions that cannot be approximated with regular neural networks (e.g. discontinuous functions).

In the standard implementation of the implicit policies (see for instance: \cite{ibc}), that we will refer to as the \textit{implicit one-tower} (IOT) policies, the state and action feature vectors are concatenated and such an input is given to the energy-network. While seemingly natural, this approach has one crucial weakness - it is prohibitively expensive for large action-spaces. It requires solving nontrivial optimization $\mathrm{argmin}$-problem at every state-to-state transition without opportunity to at least partially reuse computations conducted in the past. Indeed, even if the same actions are probed for different transitions (e.g. when the action-space has moderate cardinality or sampling heuristics are applied and the same sets of action samples are applied across different state-transitions) those actions are concatenated with different states leading to different inputs to the energy-function for different transitions. That is why in practice implicit policies apply sampling techniques with relatively few samples, affecting the approximation quality of the original $\mathrm{argmin}$-problem.

\begin{figure*}[!htb]
  \includegraphics[width=\linewidth]{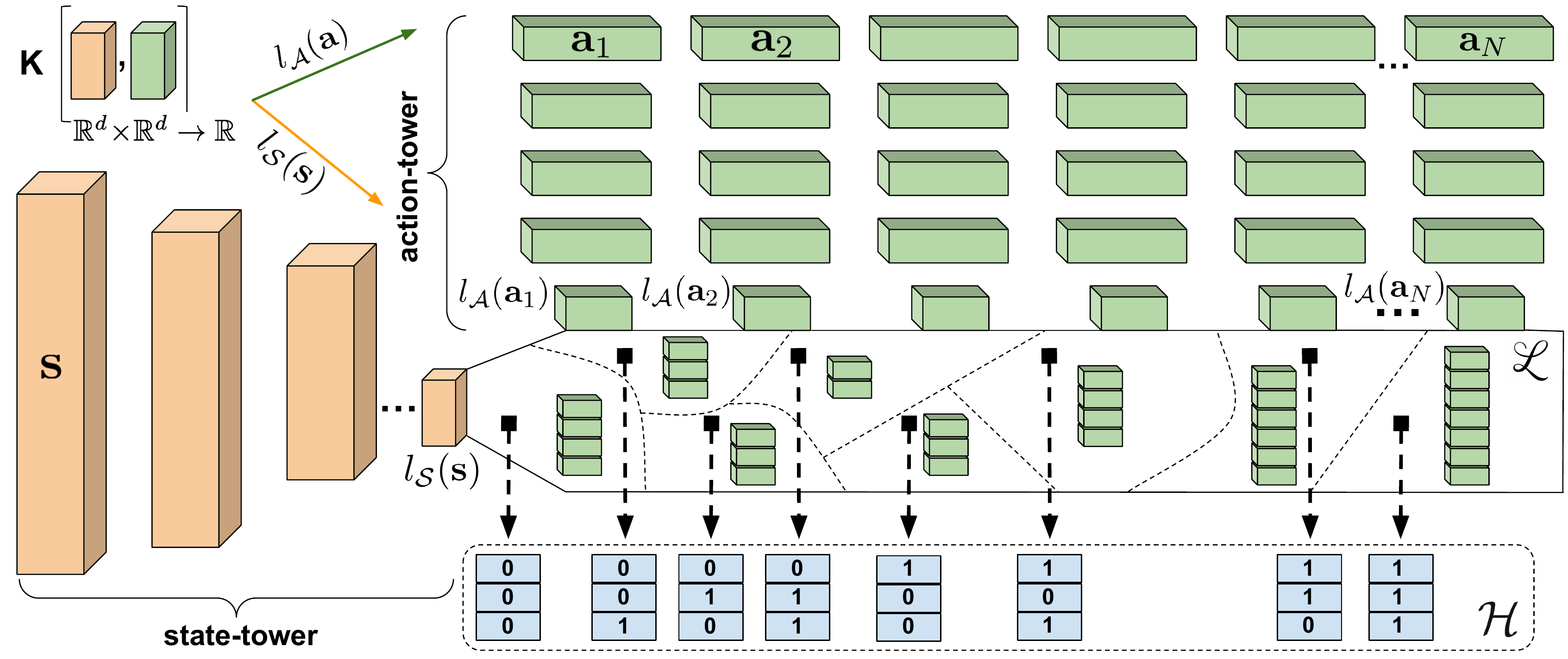}
\caption{\small{The conceptual description of the Implicit Two-Tower (ITT) stack. The (orange) state-tower and (green) action-tower map states $\mathbf{s}$ and actions $\mathbf{a}$ into their corresponding latent representations $l_{\mathcal{S}}(\mathbf{s}),l_{\mathcal{A}}(\mathbf{a}) \in \mathbb{R}^{d}$. The latent space $\mathscr{L}$ can be itself partitioned into subregions, for instance via hashing mechanisms, for the sublinear approximate state-action match. This potential partitioning would need to be periodically updated in the training process, but in principle could be frozen in inference (if a fixed set of sampled actions is being applied). The energy-function is defined via a simple kernel $\mathrm{K}:\mathbb{R}^{d} \times \mathbb{R}^{d} \rightarrow \mathbb{R}$ acting on the latent representations. Symbol $\mathcal{H}$ refers to the hash space.  }}
\label{fig:ITT-vis}
\end{figure*}

To address this key limitation, we introduce a new class of implicit policies architectures, called \textit{implicit two-towers} (ITTs) (Fig. \ref{fig:ITT-vis}), where action processing is explicitly disentangled from state processing via the architectural design. The ITT-architecture consists of two towers mapping states and actions to the same $d$-dimensional latent space $\mathcal{L}$. The negated energy $-E$ is then defined as a relatively simple kernel $\mathrm{K}:\mathbb{R}^{d} \times \mathbb{R}^{d} \rightarrow \mathbb{R}$ acting on the state/action latent representations (e.g. dot-product or softmax kernel). 
Such a representation has several computational advantages: 
\begin{compactitem}
\item it enables reusing computations for a fixed set of actions; computational gains are present even if actions are sampled at every transition since in ITTs (as opposed to IOTs) action-processing is completely disentangled from the expensive state-transitions of the environment and thus latent representations of actions can be pre-computed,
\item it allows for action and state towers to be updated in different rates in training (e.g. less-frequent \textit{lazy action-tower updates}) (see: Sec. \ref{sec:ittplus}), \item  it is compatible with various approximate sublinear-time algorithms for solving the $\mathrm{argmin}$-problem in its new two-tower form (see: Sec \ref{sec:srp}).
\end{compactitem}

ITTs can in particular apply a rich set of techniques for solving the maximum inner product (MIP) problem \cite{mip0, mip1, mip2, pham}, such as LSH-hashing, as well as algorithms conducting sub-linear softmax-sampling via linearization of the softmax kernel with random feature trees \cite{hrf, samsoft2}. Interestingly, they also produce policies obtaining larger rewards as compared to their IOT and explicit counterparts, as we demonstrate in Sec. \ref{sec:exps} and in Appendix \ref{sec:appendix_exps_details} on $15$ environments taken from $\mathrm{OpenAI}$ $\mathrm{Gym}$ and $\mathrm{DeepMind}$ $\mathrm{Control}$ $\mathrm{Suite}$. ITTs can be applied for both: discrete and continuous action-spaces.

The implicit policies are intrinsically related to several classes of methods developed for machine learning (ML) and robotics. We review some of them below.  
\paragraph{$Q$-learning:}
$Q$-learning methods (\cite{q1, dayan, doubleq, qtopt,he2023robust_tmlr,huang2023bi}), that are prominent examples of the off-policy RL algorithms, can be thought of as instantiations of the implicit policies techniques. The $Q$-function can be interpreted as the negated energy and it has a very special semantics: $Q(\mathbf{s},\mathbf{a})$ stands for the total reward obtained by an agent applying action $\mathbf{a}$ in state $\mathbf{s}$ and then following optimal policy. Consequently, the training of the (neural network) approximation $\widehat{Q}$ of $Q$ leverages the fact that $Q$ is a fixed point of the so-called \textit{Bellman operator} (\cite{bellman1, bellman2}). Furthermore, learning the $Q$-function is an off-policy process and the $\mathrm{argmin}$-defined policy is applied only after $Q$-learning is completed. The setting considered in this paper is more general - the energy $E$ lacks the semantics of the negated $Q$-function which enables us to bypass the separate off-policy training of $E$. The algorithms presented in this paper are in fact on-policy. 


\paragraph{Energy-based Models:} Implicit policies can be viewed as special instantiations of energy-based models (EBMs) (see: \cite{lecun, yangsong} for  
a comprehensive introduction to EBMs). Several impactful ML architectures have been recently reinterpreted as EBMs. Those include Transformers \cite{vaswani} with their attention modules resembling modern associative memory units (the latter being flagship examples of EBMs \cite{ramsauer} implementing differentiable dictionaries via exponential energies). We mention Transformers here on purpose. In ITTs the energy is the dot product of latent action and negated latent state. Thus, ITTs can be interpreted as learning the cross-attention between the state and action-spaces with actions corresponding to keys and states to queries.

\textbf{Our Main Contributions Are: }
\begin{compactitem}
\item We propose a new class of structured reinforcement learning architecture, implicit two-tower policies. We demonstrate the new architecture achieves stronger performance than existing implicit policies and their explicit counterparts.
\item We adapt fast maximum inner product search algorithms to solve the energy maximization problem in implicit policies. Thus, we resolve the exponential sample complexity problem for implicit policies and allow implicit policies to scale. 
\item By disentangling action and state processing, ITTs allow for state and action towers to be updated at different rates and makes it possible to reuse computations conducted for a fixed set of actions. Thus, we achieve further computational gains. 
\end{compactitem}

\section{IMPLICIT TWO-TOWER POLICIES}
\label{sec:itt}
\subsection{Preliminaries}
\label{sec:ip-preliminaries}
As described in Section \ref{sec:intro}, we focus in this paper on the implicit policies $\pi_{\theta}:\mathcal{S} \rightarrow \mathcal{A}$ from the state-space $\mathcal{S} \subseteq \mathbb{R}^{s}$ to the action-space $\mathcal{A} \subseteq \mathbb{R}^{a}$, given as follows for the learnable $\theta \in \mathbb{R}^{D}$:
\begin{equation}
\label{eq:impliciteq}
\pi_{\theta}(\mathbf{s}) = \mathrm{argmin}_{\mathbf{a} \in \mathcal{A}} E_{\theta}(\mathbf{s},\mathbf{a})    
\end{equation}
Here $E_{\theta}:\mathcal{S} \times \mathcal{A} \rightarrow \mathbb{R}$ is the \textit{energy-function}, usually encoded by the neural network. In the standard implicit-policy approach, the \textit{one-tower} model (IOT), the input to this neural network is the concatenated state-action vector: $\mathrm{input}=[\mathbf{s},\mathbf{a}]$. 
Solving optimization problem from Eq. \ref{eq:impliciteq} directly is usually prohibitively expensive. Thus instead sampling strategies are often used. For the selected set $\mathcal{A}^{*} = \{\mathbf{a}_{1},...,\mathbf{a}_{N}\}$ of sampled actions (usually uniformly at random from $\mathcal{A}$), the algorithm approximates $\pi_{\theta}(\mathbf{s})$ as: $\widehat{\pi}_{\theta}(\mathbf{s})=\mathrm{argmin}_{\mathbf{a} \in \mathcal{A}^{*}} E_{\theta}(\mathbf{s},\mathbf{a})$ or applies derivative-free-optimization heuristics (see: \cite{ibc}). Alternatively, the task is relaxed and instead of solving the original $\mathrm{argmin}$-problem, the action $\mathbf{a} \in \mathcal{A}^{*}$ is sampled from the following softmax-distribution defined on $\mathcal{A}^{*}$ (the relaxation allows backpropagating through the action-selection modules):
\begin{equation}
\label{eq:softmaxsampling}
\mathbb{P}[\widehat{\pi_{\theta}}(\mathbf{s})=\mathbf{a}_{i}] = \frac{\exp(-E_{\theta}(\mathbf{s},\mathbf{a}_{i}))}{\sum_{\mathbf{a} \in \mathcal{A}^{*}}\exp(-E_{\theta}(\mathbf{s},\mathbf{a}))}    
\end{equation}
All the aforementioned approaches are inherently linear in the number of sampled actions. Furthermore,  sampling is usually conducted at every state-transition. Thus in practice, for computational efficiency, a small number of samples needs to be used.

\subsection{Two Towers}
\label{sec:tt}
In the implicit two-tower (ITT) model the energy is defined as:
\begin{equation}
E_{\theta}(\mathbf{s},\mathbf{a}) = -\mathrm{K}(l^{\theta_{1}}_{\mathcal{S}}(\mathbf{s}),l^{\theta_{2}}_{\mathcal{A}}(\mathbf{a}))    
\end{equation}
for the state-tower mapping: $l^{\theta_{1}}_{\mathcal{S}}(\mathbf{s}):\mathbb{R}^{s} \rightarrow \mathcal{L} \subseteq \mathbb{R}^{d}$ and action-tower mapping: $l^{\theta_{2}}_{\mathcal{A}}(\mathbf{a}):\mathbb{R}^{a} \rightarrow \mathcal{L} \subseteq \mathbb{R}^{d}$, parameterized by $\theta_{1} \in \mathbb{R}^{D_{1}}$ and $\theta_{2} \in \mathbb{R}^{D_{2}}$ respectively (usually encoded by two neural networks) as well as a fixed kernel function $\mathrm{K}:\mathbb{R}^{d} \times \mathbb{R}^{d} \rightarrow \mathbb{R}$. Here $\mathcal{L}$ stands for the common latent space for states and actions.
As in the case of regular implicit policies, action selection is conducted by solving the $\mathrm{argmin}$-problem or its softmax-sampling relaxation.

A particularly prominent class of kernels that can be applied are those that are increasing functions of the dot-products of their inputs, i.e. $\mathrm{K}(\mathbf{x},\mathbf{y})=f(\mathbf{x}^{\top}\mathbf{y})$ for some $f:\mathbb{R} \rightarrow \mathbb{R}$. Those include dot-product kernel, where $f$ is an identity function as well as the softmax-kernel, where $f(z)\overset{\mathrm{def}}{=}\exp(z)$. For those kernels the corresponding $\mathrm{argmax}$-problems are trivially equivalent and reduce to the maximum-inner-product (MIP) search \cite{shrivastava2014asymmetric,neyshabur2015symmetric,sundaram2013streaming}, but the softmax-distributions differ. 
For a fixed sampled set of actions $\mathcal{A}^{*}=\{\mathbf{a}_{1},...,\mathbf{a}_{N}\}$, the MIP problem can be solved particularly efficiently as follows:
\begin{equation}
\label{eq:mip}
\widehat{\pi}_{\theta_{1},\theta_{2}}(\mathbf{s}) = \mathbf{a}_{\mathrm{argmax}\left(l^{\theta_{2}}_{\mathcal{A}}(\mathcal{A}^{*})l^{\theta_{1}}_{\mathcal{S}}(\mathbf{s})\right)},     
\end{equation}
where the $i^{th}$ rows of the matrix $l^{\theta_{2}}_{\mathcal{A}}(\mathcal{A}^{*}) \in \mathbb{R}^{N \times d}$ is given as $l^{\theta_{2}}_{\mathcal{A}}(\mathbf{a}_{i})$. Thus brute-force computation of the action for the given state between two consecutive updates of the action-tower (or: the set of sampled actions) takes time: $O(Nd + T_{\mathcal{S}})$, where $T_{\mathcal{S}}$ is the time needed to compute latent representation $l^{\theta_{1}}_{\mathcal{S}}(\mathbf{s})$ of $\mathbf{s}$.
\subsection{Fast Maximum Inner Product \& Beyond}
\label{sec:ittplus}

\subsubsection{Signed Random Projections}
\label{sec:srp}
The formulation from Equation \ref{eq:mip} is amenable to the hashing-based relaxation. In this setting set $\mathcal{A}^{*}$ is partitioned into nonempty subsets: $\mathcal{A}_{1}^{*},...,\mathcal{A}_{p}^{*}$ based on the hash-map: $h:\mathbb{R}^{d} \rightarrow \mathbb{Z}^{m}$, where $\mathbb{Z}$ is the set of all integers (a given subset of the partitioning contains actions from $\mathcal{A}^{*}$ with the same vector-value of the hash-map). Hashing techniques (e.g. locality-sensitive hashing) are applied on the regular basis to speed up nearest neighbor search (NNS). The MIP-formulation at first glance does not look like the NNS, but can be easily transformed  to the NNS-formulation (see: \cite{pham}), leading in our case to the new definitions of the latent embeddings corresponding to actions and states:
\vspace{-3mm}
\begin{align}
\begin{split}
\label{eq:mip-tp-nns}
\tilde{l}^{\theta_{2}}_{\mathcal{A}}(\mathbf{a}) = \left[(l^{\theta_{2}}_{\mathcal{A}}(\mathbf{a}))^{\top},\sqrt{C^{2}-\|l^{\theta_{2}}_{\mathcal{A}}(\mathbf{a})\|^{2}_{2}}\right]^{\top}\\  
\tilde{l}^{\theta_{1}}_{\mathcal{S}}(\mathbf{s}) = [(l^{\theta_{1}}_{\mathcal{S}}(\mathbf{s}))^{\top},0]^{\top},
\end{split}
\end{align}
where $C$ stands for the upper bound on the length of the original latent action-representation (e.g. $C=\max_{\mathbf{a} \in \mathcal{A}^{*}}\|l^{\theta_{2}}_{\mathcal{A}}(\mathbf{a})\|_{2}$). If the nonlinearity $g:\mathbb{R} \rightarrow \mathbb{R}$ used in the last layer of the action-tower satisfies: $|g(x)| \leq B$ for some finite $B>0$, then one can take: $C=B\sqrt{d}$.
\begin{figure*}[!htb]
  \includegraphics[width=\linewidth]{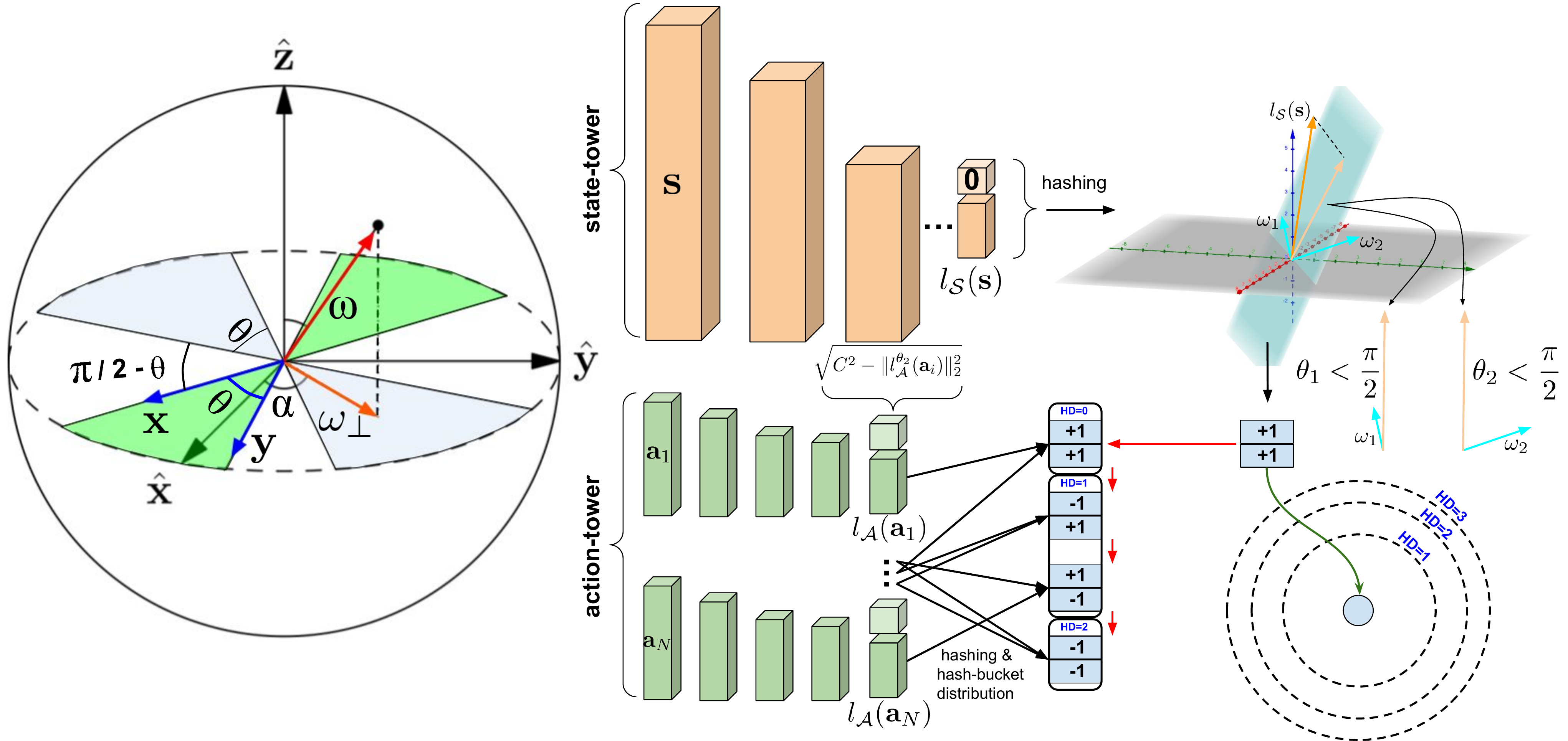}
\caption{\small{Pictorial description of the Signed Random Projections (SRP) LSH-hashing mechanism that can be applied in the ITT-model, in particular when larger sets of sampled actions are needed. \textbf{Left:} Explanation of the SRP-mechanism. SRP relies on the fact that the probability that $\mathrm{sgn}(\omega^{\top}\mathbf{x})\mathrm{sgn}(\omega^{\top}\mathbf{y}) < 0$ is the same as the probability that the projection $\omega_{\perp}$ of $\omega \sim \mathcal{N}(0,\mathbf{I}_{d})$ into the subspace spanned by  $\{\mathbf{x},\mathbf{y}\}$ forms angle $<\frac{\pi}{2}$ with one of $\{\mathbf{x},\mathbf{y}\}$ and $>\frac{\pi}{2}$ with the other one (vector $\omega_{\perp}$ inside one of the blue regions). That probability is proportional to $\theta_{\mathbf{x},\mathbf{y}}$ since $\omega_{\perp}$ is also Gaussian  \textbf{Right:} The latent representation of the state is hashed via the projections onto a random hyperplane spanned by the Gaussian vectors $\omega_{i}$. The hash's entries are determined based on the angle $\theta_{i}$ formed with vectors $\omega_{i}$ (+1 for $\theta_{i} < \frac{\pi}{2}$ and $-1$ for $\theta_{i} > \frac{\pi}{2}$). Action hash-buckets are sorted by the Hamming distance from the state's hash and the search is conducted in that order (red arrows). }}
\label{fig:ITT-srp}
\vspace{-3mm}
\end{figure*}

Note that the re-formulation from Equation \ref{eq:mip-tp-nns} preserves dot-products, i.e. we trivially have:
\begin{equation}
(\tilde{l}^{\theta_{2}}_{\mathcal{A}}(\mathbf{a}))^{\top}\tilde{l}^{\theta_{1}}_{\mathcal{S}}(\mathbf{s})=
(l^{\theta_{2}}_{\mathcal{A}}(\mathbf{a}))^{\top}l^{\theta_{1}}_{\mathcal{S}}(\mathbf{s}),
\end{equation}
but it has a critical advantage over the previous one - the latent representations of actions have now exactly the same length $L=C$. Thus the original MIP problem becomes the NNS with the angular distance. To approximate the angular distance, we will apply the Signed Random Projection (SRP) method. The method relies on the linearization of the angular kernel via random feature (RF) map mechanism \cite{unrea}. The angular kernel $\mathrm{K}_{\mathrm{ang}}:\mathbb{R}^{d} \times \mathbb{R}^{d} \rightarrow \mathbb{R}$ is defined as:
\begin{equation}
\mathrm{K}_{\mathrm{ang}}(\mathbf{x},\mathbf{y}) = 1 - \frac{2\theta_{\mathbf{x},\mathbf{y}}}{\pi},    
\end{equation}
where $\theta_{\mathbf{x},\mathbf{y}}$ is an angle between $\mathbf{x}$ and $\mathbf{y}$. The key observation is that $\mathrm{K}_{\mathrm{ang}}$ can be rewritten as:
\begin{align}
\begin{split}
\mathrm{K}_{\mathrm{ang}}(\mathbf{x},\mathbf{y}) = \mathbb{E}\left[\phi(\mathbf{x})^{\top}\phi(\mathbf{y})\right] \\
\textrm{  for   }
\phi(\mathbf{z}) \overset{\mathrm{def}}{=} \frac{1}{\sqrt{m}}(\mathrm{sgn}(\omega_{1}^{\top}\mathbf{z}),...,\mathrm{sgn}(\omega_{m}^{\top}\mathbf{z}))^{\top},
\end{split}
\end{align}
where $\omega_{1},...,\omega_{m} \overset{\mathrm{iid}}{\sim} \mathcal{N}(0,\mathbf{I}_{d})$. Thus each latent state/action representation can be mapped into the hashed space $\{-1,+1\}^{m} \subseteq \mathbb{Z}^{m}$ via the mapping: $\mathbf{z} \overset{h}{\rightarrow} (\mathrm{sgn}(\omega_{1}^{\top}\mathbf{z}),...,\mathrm{sgn}(\omega_{m}^{\top}\mathbf{z}))^{\top}$ and in that new space the search can occur based on the Hamming-distance from the hash-buckets corresponding to actions. The computational gains are coming from the fact that during that search, for a given input state $\mathbf{s}$, lots of these buckets (and thus also corresponding sets of sampled actions) will not need to be  exercised at all. In our implementation, we construct $\omega_{1},...,\omega_{m}$ such that their marginal distributions are still Gaussian (thus unbiasedness of the angular kernel estimation is maintained), yet they form a block-orthogonal ensemble. This provides additional variance reduction for any number $m$ of RFs (see: \cite{unrea}). The ITT-pipeline applying  SRPs is schematically presented in Fig. \ref{fig:ITT-srp}

\subsubsection{Random Feature Trees}
\label{sec:rft}

Let us assume that kernel $\mathrm{K}$ can be linearized as follows: $\mathrm{K}(\mathbf{x},\mathbf{y}) = \mathbb{E}[\phi(\mathbf{x})^{\top}\phi(\mathbf{y})]$ for some (potentially randomized) mapping $\phi$. Denote by $\psi$ the positive random feature map mechanism (FAVOR+) from \cite{performer} for linearizing the softmax-kernel (i.e.: $\exp(\mathbf{x}^{\top}\mathbf{y}) = \mathbb{E}[\psi(\mathbf{x})^{\top}\psi(\mathbf{y})]$).
Without loss of generality, we will assume that the size of the sampled actions set $\mathcal{A}^{*}$ satisfies: $|\mathcal{A}^{*}|=2^{k}$ for some $k \in \mathbb{N}$.
We construct a binary tree $\mathcal{T}$ with nodes corresponding to the subsets of $\mathcal{A}^{*}$. In the root we put the entire set $\mathcal{A}^{*}$. The set of actions in each non-leaf node is split into two equal-size parts (uniformly at random) and those are assigned to its two children. Leaves of the tree correspond to singleton-sets of actions.

Action assignment for a given state $\mathbf{s}$ is conducted via the binary search in $\mathcal{T}$ starting at its root.
Whenever the algorithm reaches the leaf, its corresponding action is assigned to the input state $\mathbf{s}$. Assume that the algorithm reached non-leaf node $v$ of $\mathcal{T}$. Denote its children as: $v_{l}$ and $v_{r}$ and the corresponding action-sets as $\mathcal{A}^{*}_{v_{l}}$ and $\mathcal{A}^{*}_{v_{r}}$ respectively. For the ITT architecture, the following is true:
\begin{lemma}
\begin{align*}
&\mathbb{P}[\widehat{\pi_{\theta}}(\mathbf{s})\in \mathcal{A}^{*}_{v_{l}} | \widehat{\pi_{\theta}}(\mathbf{s}) \in \mathcal{A}^{*}_{v_{l}} \cup \mathcal{A}^{*}_{v_{r}}] \\
=&\frac{\psi(\phi(l_{\mathcal{S}}^{\theta_{1}}(\mathbf{s})))^{\top}\xi(v_{l})}{\psi(\phi(l_{\mathcal{S}}^{\theta_{1}}(\mathbf{s})))^{\top}\left(\xi(v_{l})+\xi(v_{r})\right)},
\end{align*}
where $\xi(v) \overset{\mathrm{def}}{=} \sum_{a \in \mathcal{A}^{*}_{v}} \psi(\phi(l^{\theta_{2}}_{\mathcal{A}}(\mathbf{a})))$. 
\label{lem:approxSoftmaxSampling}
\end{lemma}
The proof is relegated to the appendix. 
We conclude that in order to decide whether to choose $v_{l}$ or $v_{r}$, the algorithm just needs to sample from the binary distribution with: $p_{l}=\frac{a}{a+b}, p_{r}=\frac{b}{a+b}$, where $a=\psi(\phi(l_{\mathcal{S}}^{\theta_{1}}(\mathbf{s})))^{\top}\xi(v_{l})$, $b=\psi(\phi(l_{\mathcal{S}}^{\theta_{1}}(\mathbf{s})))^{\top}\xi(v_{r})$. Thus if $\xi(v)$ is computed for every node, this sampling can be conducted in time constant in $N$. 

The total time of assigning the action to the input state is $O(\log(N))$ (rather than linear)  in the number of sampled actions (but linear in the number of random features). In practice, the actions do not need to be stored explicitly in the tree and in fact even the tree-structure does not need to be stored explicitly. We call the above tree the \textit{Random Feature Tree} (or RFT) (see also: \cite{hrf, samsoft2}).


\paragraph{Lazy Action-tower Ppdates:} Both considered data structures: SRP- and RFT-based hashes need to be updated every time the parameters of the action-tower are updated, but provide desired speedups between consecutive updates (if the sets of chosen sampled actions do not change). Fortunately, in the ITT-model, the frequency of updates of the action-tower can be completely disentangled from the one for the state-tower. In particular, the action-tower can be updated much less frequently or with frequency decaying in time.
\section{ES-OPTIMIZATION SETUP}
\label{sec:es}
We decided to train parameters of our ITT-architectures with the class of Evolutionary Search (ES, Blackbox) methods (\cite{es, toeplitz_pl, srs, asebo}). Even though ITTs are agnostic to the particular training algorithm, ES-methods constituted a particularly attractive option. They enabled us to benchmark implicit policies in the on-policy setting, where to the best of our knowledge they were never tried before. Furthermore, as embarrassingly simple conceptually (yet very efficient at the same time), they let us focus on the architectural aspects rather than tedious hyperparameter-tuning. Finally, they work very well also with non-differentiable or even non-continuous objectives, fitting well the combinatorial-flavor of the energy optimization problem formulation in ITTs.

Let $\theta = (\theta_1^\top,\theta_2^\top)^\top\in\mathbb{R}^{D}$, where $D=D_1+D_2$ and $\theta_{i} \in \mathbb{R}^{D_{i}}$ for $i=1,2$. 
We are allowed to query an objective $F(\cdot)$ that measures the expected discount cumulative reward of policy parameters $\theta\in\mathbb{R}^D$. The objective is commonly used in the RL literature and can be evaluated by running a trajectory with $\theta$ \cite{toeplitz_pl,sutton2018reinforcement,dou2022sampling,zhou2023value,he2022robust_iros,huang2023adaptive}. We conducted gradient-based optimization with the antithetic ES-gradient estimator applying orthogonal samples (see: \cite{toeplitz_pl}), defined as follows:
\begin{align}
\label{eq:orthogonal_at_gradient}
\widehat{\nabla}_{M}^{\mathrm{AT}, \text { ort }} F_{\sigma}(\theta)
&\coloneqq
\frac{1}{2 \sigma M} \sum_{i=1}^{M} F^{AT(i)}\\
\quad\text{where}\ \ F^{AT(i)} 
&\coloneqq F\left(\theta+\sigma \ve_{i}\right) \ve_{i}-F\left(\theta-\sigma \ve_{i}\right) \ve_{i}\nonumber
\end{align}
for the hyper-parameter $\sigma>0$ and where $\left(\varepsilon_{i}\right)_{i=1}^{M}$ have marginal distribution $\mathcal{N}(\mathbf{0}, I_{D})$, and $\left(\varepsilon_{i}\right)_{i=1}^{M}$ are conditioned to be pairwise-orthogonal, and $M$ is the number of perturbations we choose. Such an ensemble of samples can always be constructed since in all our experiments we have: $M \leq D$.
To be more specific, for the DMCS environments we used: $M=500$ and for all other: $M=D$.
At each training epoch, we were updating $\theta$ as 
\begin{align}
\label{eq:theta_update}
\theta_{k+1}=\theta_k+\eta \widehat{\nabla}_{M}^{\mathrm{AT}, \text { ort }}F_{\sigma}(\theta),
\end{align}
where the learning rate $\eta=0.01$ is fixed throughout the experiments.  

{\bf Theoretical Results on ES-optimization:} 
We present in Appendix~\ref{sec:theoretical_results} theoretical analysis of the antithetic estimator , to shed light on its effectiveness (see Theorem~\ref{thm:mse_at_fd} and Lemma~\ref{lem:mse_at_fd}). We will now focus on discussing our main experimental results. 
\section{EXPERIMENTS}
\label{sec:exps}
We conducted three experiments: (a) compared ITT with IOT and explicit policies on 15 environments from $\mathrm{OpenAI}$ $\mathrm{Gym}$ and $\mathrm{DeepMind}$ $\mathrm{Control}$ $\mathrm{Suite}$ (DMCS) RL libraries \cite{brockman2016openai, tassa2018deepmind}, (b) compared regular ITT with its SRP-version, (c) compared regular ITT with its lazy-tower version (the latter two on the subset of these environments).  We present neural network specifications and hyper-parameters in  Appendix~\ref{subsec:appendix_hyperparameters}. \textbf{Our main findings are: }
\vspace{-3mm}
\begin{itemize}
\item (Section 4.1) Our ITT architecture achieves stronger performance than existing implicit policies and their explicit counterparts. 

\item (Section 4.2) Our ITT architecture can be used with sublinear time inner-product search algorithms, resolving the exponential sample complexity issue with existing implicit policies, with little sacrifice in performance. 

\item (Section 4.3) We  shows that our ITT architecture allows for state and action towers to be updated at different rates, with little sacrifice in performance. Thus, we can achieve further computational gains.
\end{itemize}

\subsection{ITTs Versus IOTs and Explicit Policies}
\label{sec:exps:itt}
For the $\mathrm{OpenAI}$ $\mathrm{Gym}$ environments,  at each transition we were sampling a set of actions $\mathcal{A}^{*}=\{\mathbf{a}_{1},...,\mathbf{a}_{N}\}$ (see Equation~\ref{eq:theta_update}), and choosing an action according to Equation~\ref{eq:mip}. For the DMCS environments, $\mathcal{A}^{*}$ was sampled at each iteration of the ES-optimization (independently for different ES-workers). For the \underline{environments with discrete actions}, such as MountainCar-v0, we choose $\mathcal{A}^{*}=\mathcal{A}$. For the \underline{environments with continuous actions}, we set $N=1000$ for the $\mathrm{OpenAI}$ $\mathrm{Gym}$ environments and $N=10000$ for the DMCS environments. Figure~\ref{fig:baseModel}, Figure ~\ref{fig:baseModel_appendix}, and Table~\ref{table:DMCS_additional_envs} (the later two are in Appendix~\ref{subsec:appendix:additional_envs}) illustrate the results. ITTs provide the best policies on 12 out of 15 tasks in terms of the final average score and is second best on the remaining three, while having substantially faster iterations than IOTs (see: wall-clock time analysis below). For the environments from Table~\ref{table:DMCS_additional_envs} the training curves were much less monotonic than for the other ones thus, for the clarity of the presentations, the corresponding results were given in the table. 
\begin{table}[h!]
\caption{\small{Setting analogous to the one from Fig. \ref{fig:baseModel}. We present final average scores over $s=10$ random seeds together with their std. The best architecture is in bold font and the second best is underscored.}}
\label{table:DMCS_additional_envs}
\vspace{2mm}
 \scalebox{.95}{
    \label{tab:table1}
    \begin{tabular}{l|c|c|c} 
    \textbf{Task} & \textbf{ITT} & \textbf{IOT} & \textbf{Explicit}\\
      \hline
      Swimmer6& $\textbf{988.3}_{\pm 10.1}$ & $\underline{984.9}_{\pm 28.1}$ & $767.4_{\pm 120.2}$\\
      Swimmer15&  $\textbf{995.8}_{\pm 7.0}$  & $\underline{990.7}_{\pm 0.001}$ & $936.0_{\pm 48.7}$\\
      FishSwim& $\textbf{652.2}_{\pm 152.1}$ & $331.7_{\pm 8.4}$ & $\underline{470.7}_{\pm 0.004}$\\
    \end{tabular}
}
\end{table}

\begin{figure*}[!htb]
\minipage{0.32\textwidth}
  \includegraphics[width=\linewidth]{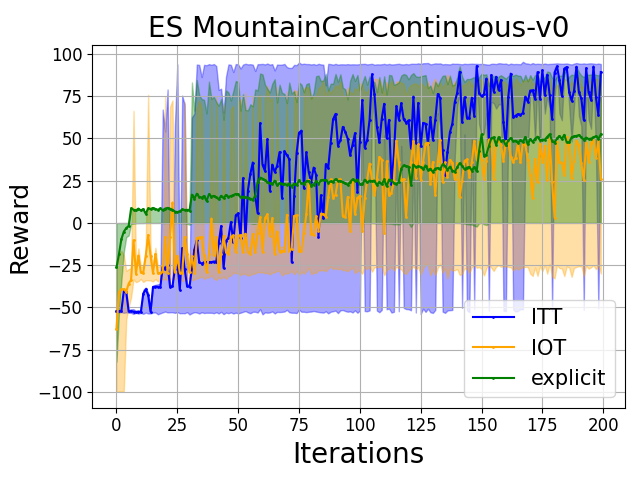}
\endminipage\hfill
\minipage{0.32\textwidth}
  \includegraphics[width=\linewidth]{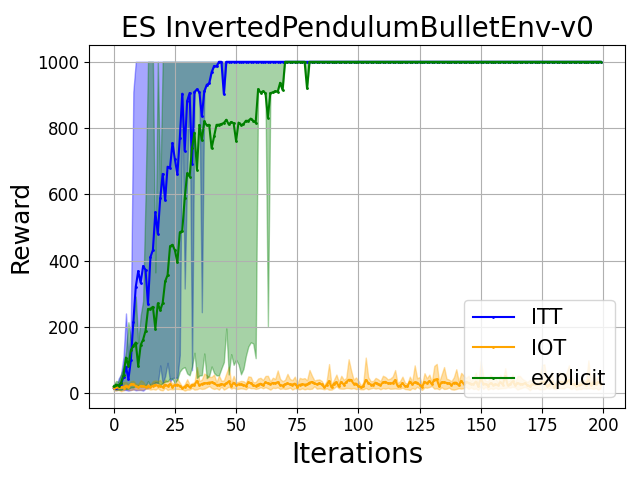}
\endminipage\hfill
\minipage{0.32\textwidth}%
  \includegraphics[width=\linewidth]{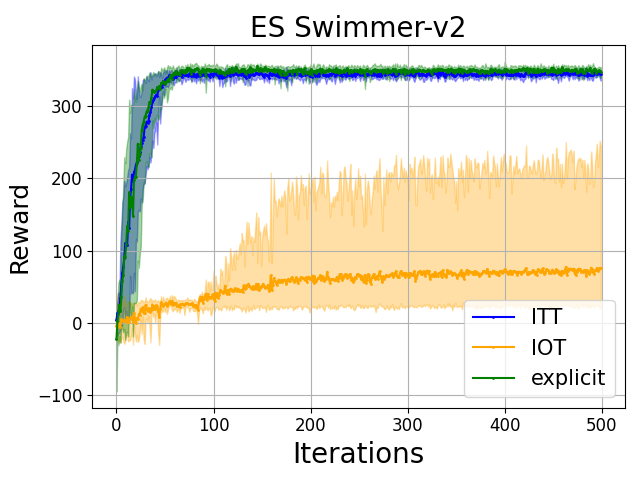}
\endminipage
\end{figure*}

\begin{figure*}[!htb]
\minipage{0.32\textwidth}
  \includegraphics[width=\linewidth]{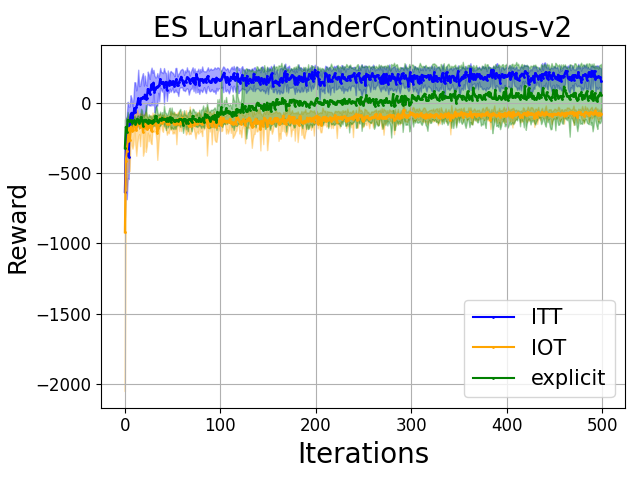}
\endminipage\hfill
\minipage{0.32\textwidth}
  \includegraphics[width=\linewidth]{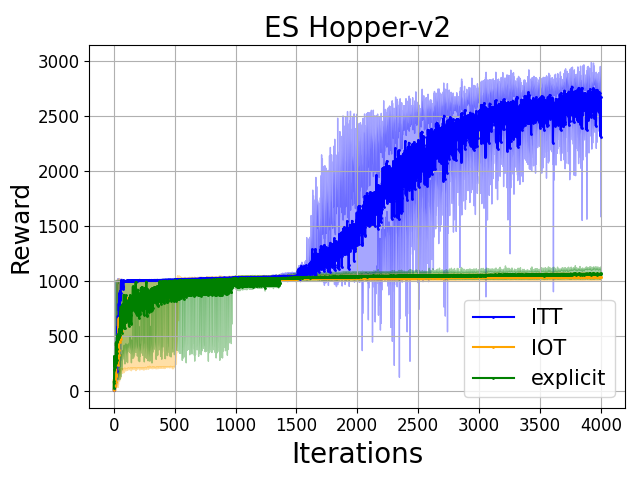}
\endminipage\hfill
\minipage{0.32\textwidth}%
  \includegraphics[width=\linewidth]{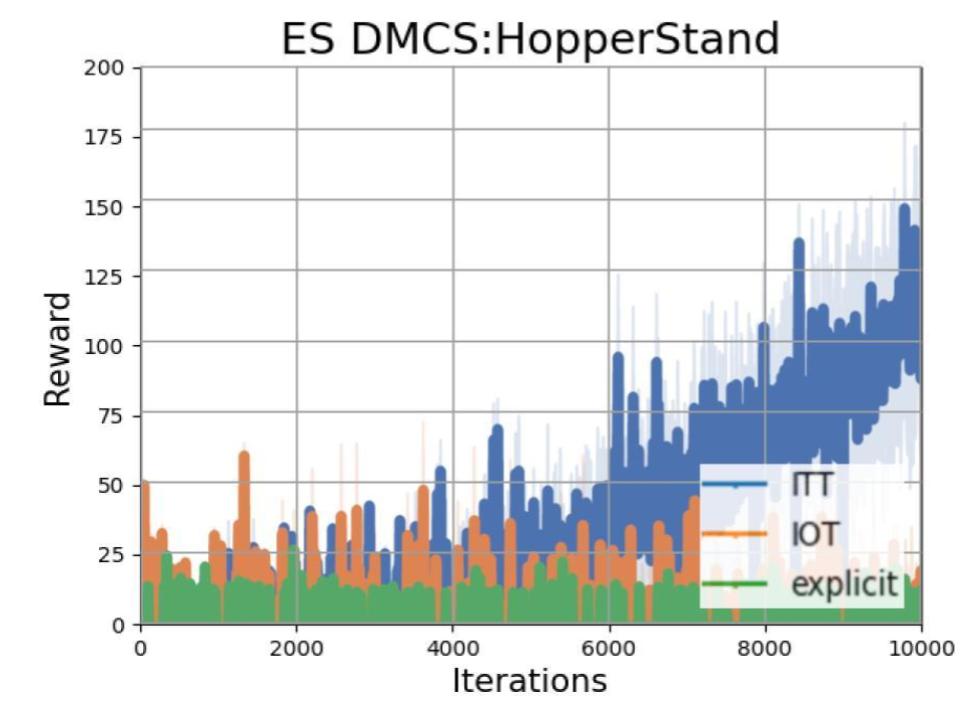}
\endminipage
\end{figure*}

\begin{figure*}[!htb]
\minipage{0.32\textwidth}
  \includegraphics[width=\linewidth]{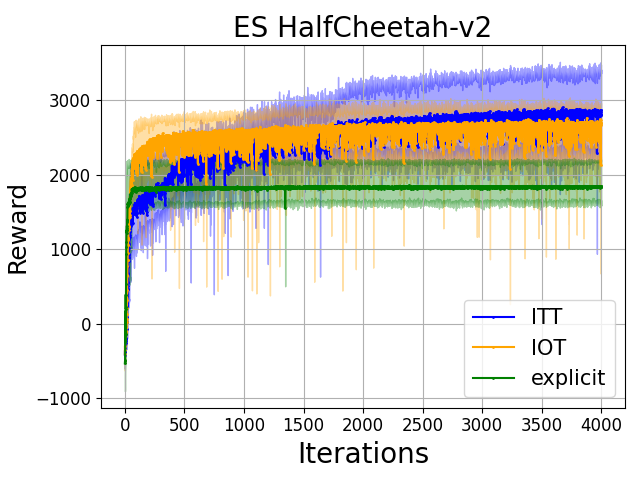}
\endminipage\hfill
\minipage{0.32\textwidth}
  \includegraphics[width=\linewidth]{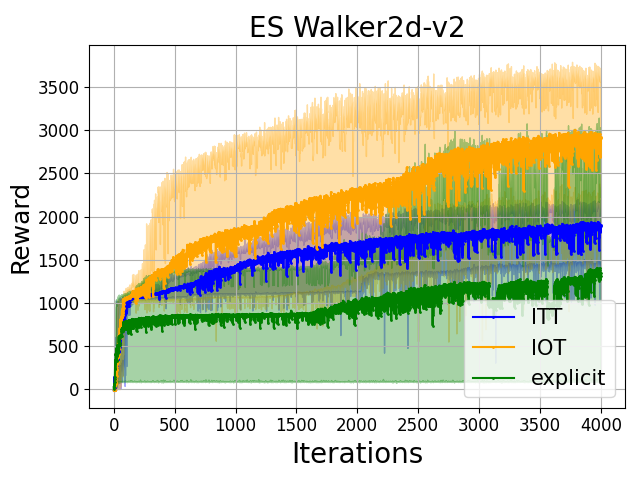}
\endminipage\hfill
\minipage{0.32\textwidth}%
  \includegraphics[width=\linewidth]{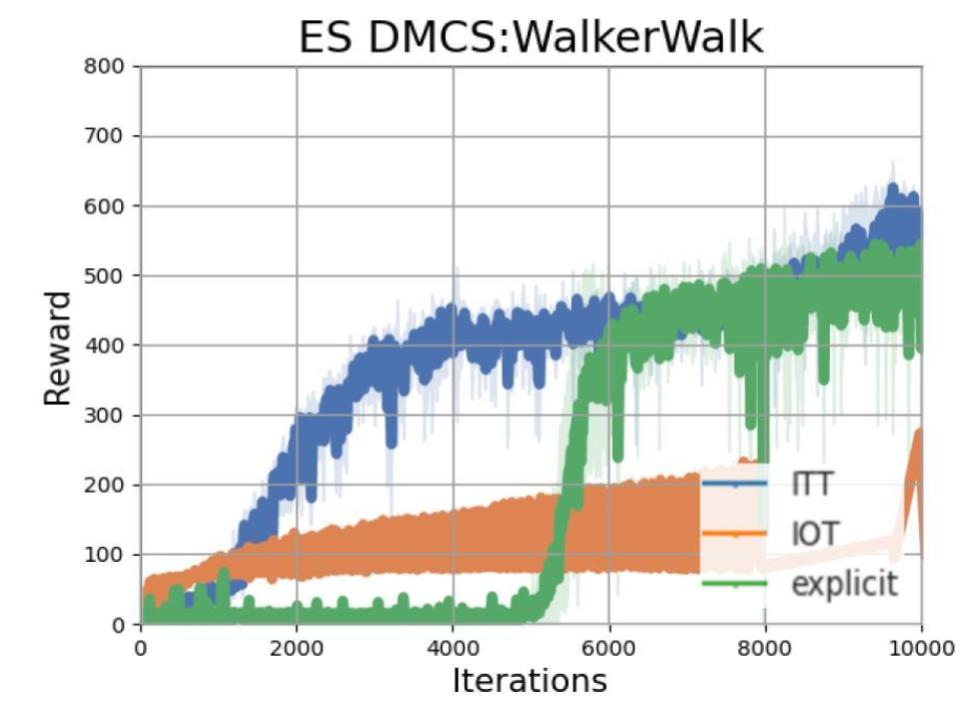}
\endminipage
\caption{\small{The comparison of the performance of different policy-architectures: ITTs, IOTs and explicit on various $\mathrm{OpenAI}$ $\mathrm{Gym}$ and $\mathrm{DeepMind}$ $\mathrm{Control}$ $\mathrm{Suite}$ tasks. We plot average curves obtained from $s=10$ seeds, and present the $90^{th}$ percentile and the $10^{th}$ percentile using shadowed regions. For fair comparison, we choose architectures' sizes such that the number of trainable parameters of ITTs is comparable to but upper-bounded by that of IOTs and explicit variants (see: Appendix~\ref{sec:appendix_exps_details} with additional environments and experimental details).}}
\label{fig:baseModel}
\vspace{-3mm}
\end{figure*}


\textbf{Compare to Additional Baselines}.
We provide comparison to previous works on different OpenAI Gym environments in Table 7 in Appendix~\ref{subsec:appendix_baselines}. 
For DMCS environments, we note the scores achieved by ITT are competitive (see Table 1, p.21 in \cite{tassa2018deepmind}, which provides exhaustive comparison of different non-ES methods on Deep Mind Control Suite). We provide more detailed discussions in Appendix~\ref{subsec:appendix_baselines}.

{\bf Wall-clock Time.} Table~\ref{table:walltime} shows the wall-clock time (in hours on 24 CPU cores) for each policy-architecture for a selected set of environments. For each of these environments,  we also include the number of training iterations. We conclude that for 
environments with lower dimensional action space
(with the corresponding smaller sizes of policy-architectures) all three architectures perform similarly speed-wise. However for more complicated environments ITTs train much faster than IOTs (\textbf{26\%} training time reduction for Hopper-v2, \textbf{45\%} for HalfCheetah-v2 and \textbf{39\%} for Walker2d-v2).


\begin{table}[h!]
\caption{\small{Comparison of the wall-clock time (in hours) for different policy-architectures and a selected sample of the environments from Fig. \ref{fig:baseModel}. "LunarLanderC-v2`` here stands for LunarLanderContinuous-v2}} \label{table:walltime}
\vspace{3mm}
 \scalebox{.9}{
    \label{tab:table1}
    \begin{tabular}{l|c|c|c|c} 
      \textbf{Environment} & \textbf{ITT} & \textbf{IOT} & \textbf{Explicit} & \textbf{\# iter}\\
      \hline
      Swimmer-v2& 0.17 & 0.18 & 0.11 & 500\\
      LunarLanderC-v2& 0.06 & 0.06 & 0.08 & 500\\
      Hopper-v2& 2.49 & 3.36 & 2.42 & 4000\\
      HalfCheetah-v2 & 12.94 & 23.61 & 9.95 & 4000\\ 
      Walker2d-v2& 16.88 & 27.49 & 13.41 & 4000\\
    \end{tabular}
  }
\end{table}



\subsection{Fast ITTs}
\label{sec:exps_srp}

\begin{figure*}
\centering
 \scalebox{.9}{
     \begin{tabular}{cc}
        \includegraphics[width=0.4\textwidth]{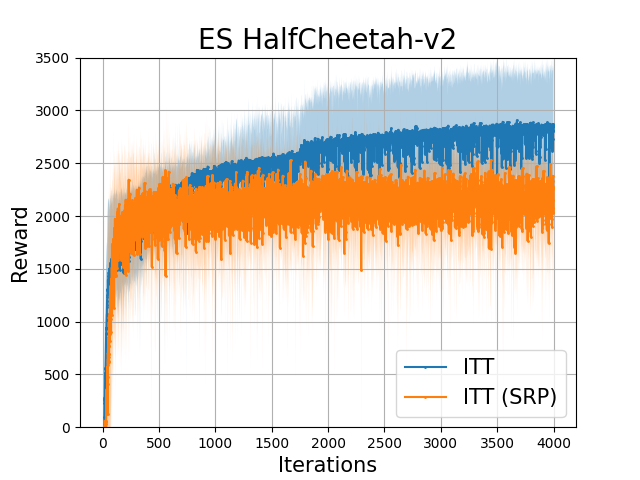} &
        \includegraphics[width=0.4\textwidth]{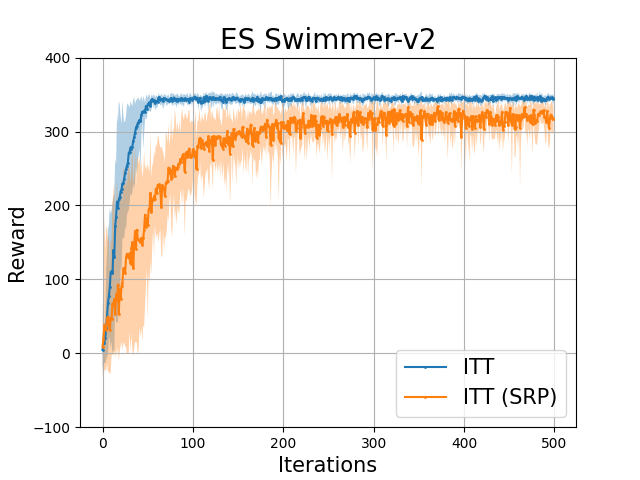} 
     \end{tabular}
     }
    \caption{\small{Comparison of vanilla ITT and ITT with Signed Random Projection for fast MIP. We plot average curves obtained from 10 seeds, and we present the $90^{th}$ percentile and the $10^{th}$ percentile as shadowed regions. We present results on {\bf additional tasks} in Appendix~\ref{subsec:appendix:additional_envs_itt_srp}.}}
    \label{fig:hash}
\vspace{-4mm}
\end{figure*}

\begin{figure*}
\centering
 \scalebox{.9}{
     \begin{tabular}{cc}
        \includegraphics[width=0.4\textwidth]{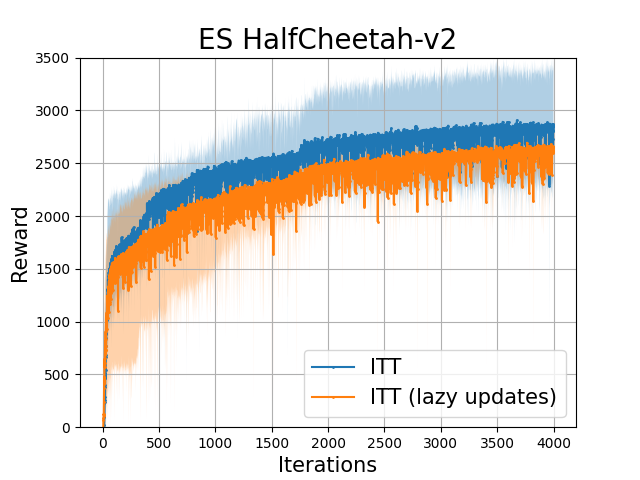} &
        \includegraphics[width=0.4\textwidth]{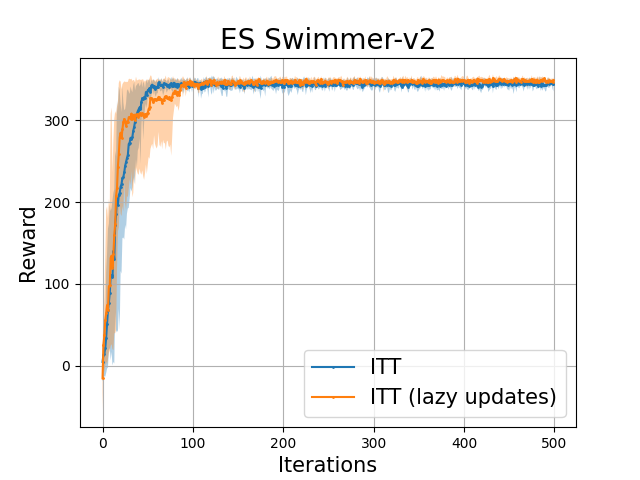} 
     \end{tabular}
     }
    \caption{\small{Comparison of vanilla ITT and ITT with lazy action-tower updates, where actions-towers are updated once every $5$ iterations. We plot average curves obtained from $s=10$ seeds, and present the $90^{th}$ percentile and the $10^{th}$ percentile using shadowed regions.}}
    \label{fig:lessUpdate}
\vspace{-4mm}
\end{figure*}

{\bf ITT-SRP (signed random projections): } 
The random projection vectors $\omega_{1},...,\omega_{m}$ have marginal distribution $\mathcal{N}(0,\mathbf{I}_{d})$ and are conditioned to be orthogonal. The orthogonality is obtained via the Gram-Schmidt orthogonalization of the iid samples (see: \cite{unrea}).  When $m$ is small, to avoid having too many or too few actions in each action hash-bucket, we calculate $b_i = \text{median}(\{\omega_i^\top \tilde{l}^{\theta_{1}}_{\mathcal{A}}(\mathbf{a_j})\}_{j=1}^N)$ for each projection vector $\omega_i$, and map the action to the hashed space using $\mathbf{z} \overset{h}{\rightarrow} (\mathrm{sgn}(\omega_{1}^{\top}\mathbf{z}-b_1),...,\mathrm{sgn}(\omega_{m}^{\top}\mathbf{z}-b_m))^{\top}.$ 
Figure~\ref{fig:hash} shows the comparison of the regular ITTs with ITTs applying SRPs. For HalfCheetah-v2, we use $N=2^{14}=16,384$  and $m=6$. For Swimmer-v2, we use $N=2^{10} =1,024$ and $m=3$. ITT-SRPs maintain good performance even though the hash-space is very small (of size $2^{8}=256$ for HalfCheetah-v2 and $2^{7}=128$ for Swimmer-v2).
 We chose $C=\max_{\mathbf{a} \in \mathcal{A}^{*}}\|l^{\theta_{2}}_{\mathcal{A}}(\mathbf{a})\|_{2}$ (see: Eq. \ref{eq:mip-tp-nns}). We present results on {\bf additional tasks} in Appendix~\ref{subsec:appendix:additional_envs_itt_srp}.
 

\textbf{ITT-RFT (random feature trees):} Similar to ITT-SRP, ITT-RFT could address exponential sample complexity issue of implicit policies, with little sacrifice in performance (see Appendix~\ref{subsec:appendix:itt-rft-results}).




{\bf ITT-lazy: } we update the action tower every $5$ iterations. Figure~\ref{fig:lessUpdate} shows the performance of this lazy variant is comparable to that of regular ITT, while the lazy action-tower variant achieves substantial wall-clock time savings (see Appendix~\ref{subsec:appendix:lazy_wall_clock_time_savings}). 

{\bf Ablation Studies: } In Appendix~\ref{sec:appendix_ablation}, we provide extensive ablation studies, showing that (1) ITT consistently outperforms even when we significantly reduce the number of neurons; (2) ITT-SRP achieves substantial reduction in wall-clock time, especially when the number of samples required gets large; (3) ITT-SRP can be applied in either training or inference or in both. 

{\bf Pair T-test Results: } In Appendix~\ref{sec:appendix_t_test}, we provide evidence of statistical significance that ITT achieves higher scores than IOT and explicit policies.  

\textbf{Compute resources.} We use HP Enterprise XL170r server with 24 cpu cores and 128GB RAM.
\section{Conclusion}
\label{sec:conclusion}
We presented in this paper a new model for the architectures of the implicit policies, called the Implicit Two-Tower (ITT) model. ITTs provide substantial computational benefits over their regular Implicit One-Tower (IOT) counterparts, yet at the same time they lead to more accurate models (also as compared to the explicit policies). They are also compatible with various hashing techniques providing additional computational gains, especially if very large sets of sampled actions are needed.


\bibliography{references}

\newpage
\appendix
\onecolumn
{
\begin{center}
    \Large
    \textbf{{Appendix: Implicit Two-Tower Policies}}
\end{center}
}

\section{ADDITIONAL EXPERIMENTAL DETAILS}
\label{sec:appendix_exps_details}
We include the code and instructions how to use it at https://anonymous.4open.science/r/itt-9881/README.md

\subsection{Neural Network Specifications and Hyper-parameter Tuning}

\textbf{Neural Network Specifications}
\label{subsec:appendix_hyperparameters}
 Tables~\ref{table:num_params}\nobreakdash--\ref{table:num_layers} illustrate  the dimension of the learnable $\theta \in \mathbb{R}^{D}$ and the number of layers in neural networks respectively. The dimensionality of the latent state and action vector as well as the dimensionalities of the hidden layers are set to the dimensionality of the action vector for the $\mathrm{OpenAI}$ $\mathrm{Gym}$ environments and are equal to $20$ for the DMCS environments. We do not use bias terms. We apply Relu activation for all the hidden layers and linear activation on the output layers, with one exception: for the Swimmer-v2 environment, we use linear activation in all layers and for all the methods (because in this environment only, using linear activation in all layers improves the performance of the baselines).

 \textbf{Hyperparameter Tuning} Tables~\ref{table:sigma} illustrate fine-tuned hyper-parameter $\sigma$, which controls the exploration in ES-optimization. For each policy-architecture (ITT, IOT, and explicit) and for each environment, we choose the value of $\sigma\in[0.1, 0.5, 1]$ that provide the highest final average score at the end of the horizon. Similarly, we choose the number of neural network layer in $[1,2,3,4,5,6]$ that provide the highest final average score at the end of the horizon. For fair comparison, we ensure the number of trainable parameters of ITTs is upper-bounded by that of IOTs and explicit variants.

\begin{table}[h!]
  \begin{center}
    \begin{tabular}{l|c|c|r} 
      \textbf{Environment} & \textbf{ITT} & \textbf{IOT} & \textbf{Explicit}\\
      \hline
      Swimmer-v2& 1 & 1 & 1\\
      LunarLanderContinuous-v2& 1 & 1 & 1\\
      Hopper-v2& 1 & 1 & 1\\
      HalfCheetah-v2 & 1 & 1 & 0.5\\
      Walker2d-v2& 0.5 & 0.5 & 0.5\\
      CartPole-v1 &  1 & 1 & 1 \\
      MountainCar-v0 &  1 & 1 & 1 \\
      Acrobot-v1 &  1 & 1 & 1 \\
      MountainCarContinuous-v0 &  1 & 1 & 1 \\
      InvertedPendulumBulletEnv-v0 &  1 & 1 & 1 \\
      DMCS:FishSwim& 0.1 & 0.1 & 0.1\\
      DMCS:Swimmer6& 0.1 & 0.1 & 0.1\\
      DMCS:Swimmer15& 0.1 & 0.1 & 0.1\\
      DMCS:HopperStand& 0.1 & 0.1 & 0.1\\
      DMCS:WalkerWalk& 0.1 & 0.1 & 0.1\\
    \end{tabular}
  \end{center}
\caption{Fine-tuned hyper-parameter $\sigma$, used in ES gradient estimator calculations (Equation~\ref{eq:orthogonal_at_gradient}) for different environments and different policy-architectures.}
\label{table:sigma}
\end{table}

\begin{table}[h!]
  \begin{center}
    \begin{tabular}{l|c|c|r} 
      \textbf{Environment} & \textbf{ITT} & \textbf{IOT} & \textbf{Explicit}\\
      \hline
      Swimmer-v2& 20 & 22 & 20\\
      LunarLanderContinuous-v2& 20 & 22 & 20\\
      Hopper-v2& 42 & 45 & 42\\
      HalfCheetah-v2 & 282 & 288 & 282\\
      Walker2d-v2& 246 & 252 & 246\\
      CartPole-v1 &  6 & 7 & 6 \\
      MountainCar-v0 &  4 & 5 & 4 \\
      Acrobot-v1 &  8 & 9 & 8 \\
      MountainCarContinuous-v0 &  3 & 4 & 3 \\
      InvertedPendulumBulletEnv-v0 &  12 & 13 & 12 \\
      DMCS:FishSwim& 2180 & 2200 & 2180\\
      DMCS:Swimmer6& 2200 & 2220 & 2200\\
      DMCS:Swimmer15& 3100 & 3120 & 3100\\
      DMCS:HopperStand& 1980 & 2000 & 1980\\
      DMCS:WalkerWalk& 2200 & 2220 & 2200\\
    \end{tabular}
  \end{center}
\caption{The dimensionality of the learnable $\theta \in \mathbb{R}^{D}$ for different environments and different policy-architectures.}
\label{table:num_params}
\end{table}

\begin{table}[h!]
  \begin{center}
    \begin{tabular}{l|c|c|c|r} 
      \textbf{Environment} & \textbf{ITT state tower} & \textbf{ITT action tower} & \textbf{IOT} & \textbf{Explicit}\\
      \hline
      Swimmer-v2& 1 & 1 & 2 &2\\
      LunarLanderContinuous-v2& 1 & 1 & 2 & 2\\
      Hopper-v2& 1 & 1& 2 & 2\\
      HalfCheetah-v2 & 4 & 2 & 6 & 6\\
      Walker2d-v2& 3 & 2 & 5 & 5\\
      CartPole-v1 &  2 & 1 & 3 &3 \\
      MountainCar-v0 &  2 & 1 & 3 &3 \\
      Acrobot-v1 &  2 & 1 & 3 &3 \\
      MountainCarContinuous-v0 &  1 & 1 & 2 &2 \\
      InvertedPendulumBulletEnv-v0 &  1 & 1 & 2 &2 \\
      DMCS:FishSwim& 3 & 1 & 5 & 5\\
      DMCS:Swimmer6& 3 & 1 & 5 & 5\\
      DMCS:Swimmer15& 3 & 1 & 5 & 5\\
      DMCS:HopperStand& 3 & 1 & 5 & 5\\
      DMCS:WalkerWalk& 3 & 1 & 5 & 5\\
    \end{tabular}
  \end{center}
\caption{The number of layers of the neural networks encoding different architectures for different environments.}
\label{table:num_layers}
\end{table}

\subsection{ITT Base Variant Results on Additional Environments}
\label{subsec:appendix:additional_envs}
In Table~\ref{table:appendix_additional_envs}, we present final average scores together with their correesponding standard deviations (over $s=10$ different seeeds) for all $15$ environments tested in the paper. In Figure~\ref{fig:baseModel_appendix}, we provide plots showing training curves for additional OpenAI Gym tasks.

\begin{table}[h!]
  \begin{center}
    \label{tab:scores}
    \begin{tabular}{l|c|c|r} 
      \textbf{Environment} & \textbf{ITT} & \textbf{IOT} & \textbf{Explicit}\\
      \hline
      Swimmer-v2& $\underline{344.13} \pm 5.18$ & $75.54 \pm 88.05$ & $\textbf{347.67} \pm 5.74$\\
      LunarLanderContinuous-v2& $\textbf{157.38} \pm 71.81$ & $-72.72 \pm 26.98$ & $\underline{62.85} \pm 176.36$\\
      Hopper-v2& $\textbf{2670.37} \pm 225.53$ & $1036.52 \pm 29.16$ & $\underline{1060.89} \pm 40.71$\\
      HalfCheetah-v2 & $\textbf{2866.02} \pm 416.81$ & $\underline{2696.29} \pm 274.51$ & $1845.64 \pm 441.96$\\
      Walker2d-v2& $\underline{1897.87} \pm 498.86$ & $\textbf{2909.85} \pm 621.45$ & $1346.93 \pm 906.33$\\
      CartPole-v1 &  $\textbf{500.00} \pm 0.00$ & $\textbf{500.00} \pm 0.00$ & $\textbf{500.00} \pm 0.00$ \\
      MountainCar-v0 & $\textbf{-133.50} \pm 27.10$ & $-200.00 \pm 0.00$ & $\underline{-143.40} \pm 37.12$ \\
      Acrobot-v1 & $\underline{-82.60} \pm 11.93$ & $\textbf{ -82.00} \pm 8.54$ & $-92.10 \pm 21.41$ \\
      MountainCarContinuous-v0 & $\textbf{89.17} \pm 8.97$ & $25.96 \pm 53.61$ & $\underline{52.34} \pm 42.74$ \\
      InvertedPendulumBulletEnv-v0 & $\textbf{1000.00} \pm 0.00$ & $ 27.10 \pm 12.70$ & $\textbf{1000.00} \pm 0.00$ \\
      DMCS:Swimmer6& $\textbf{988.26} \pm 10.11$ & $\underline{984.88} \pm 28.05$ & $767.44 \pm 120.22$\\
      DMCS:Swimmer15&  $\textbf{995.81} \pm 7.02$  & $\underline{990.68} \pm 0.001$ & $935.99 \pm 48.69$\\
      DMCS:FishSwim& $\textbf{652.21} \pm 152.07$ & $331.73 \pm 8.42$ & $\underline{470.66} \pm 0.004$\\
    \end{tabular}
  \end{center}
\caption{\small{Setting analogous to the one from Fig. \ref{fig:baseModel}. We present final average scores} over $s=10$ random seeds together with their std. The best architecture is in bold font and the second best is underscored.}
\label{table:appendix_additional_envs}
\end{table}

\begin{figure*}[!htb]
\minipage{0.32\textwidth}
  \includegraphics[width=\linewidth]{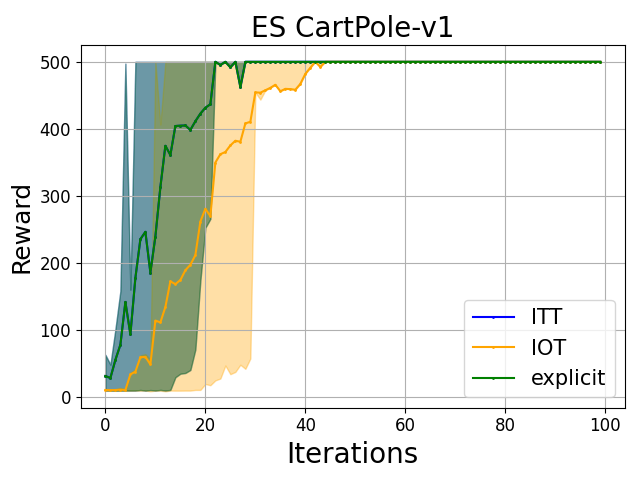}
\endminipage\hfill
\minipage{0.32\textwidth}
  \includegraphics[width=\linewidth]{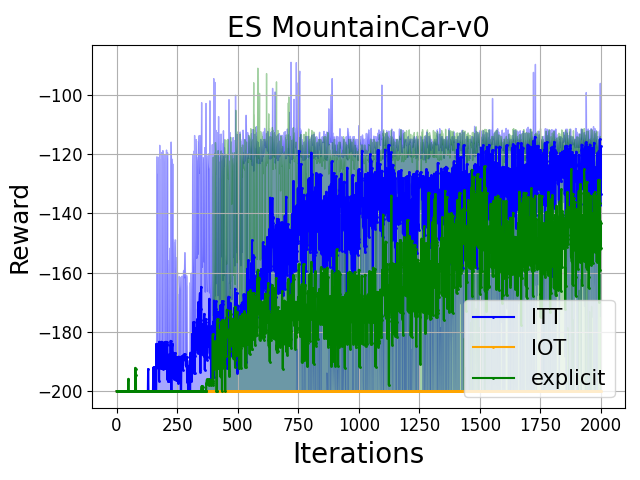}
\endminipage\hfill
\minipage{0.32\textwidth}%
  \includegraphics[width=\linewidth]{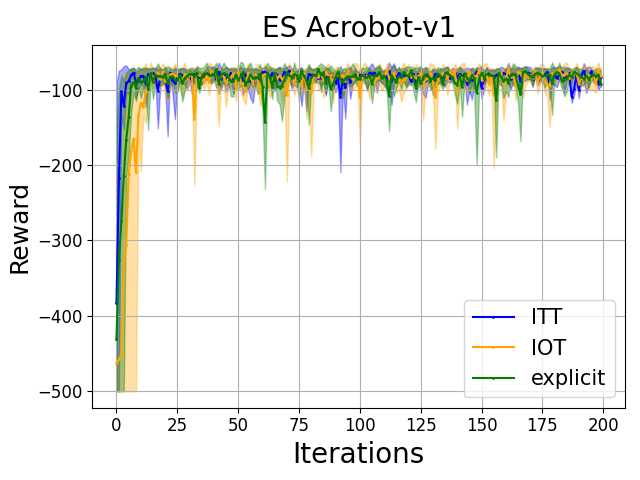}
\endminipage
\caption{\small{The comparison of the performance of different policy-architectures: ITTs, IOTs and explicit on additional $\mathrm{OpenAI}$ $\mathrm{Gym}$ tasks. We plot average curves obtained from $s=10$ seeds, and present the $90^{th}$ percentile and the $10^{th}$ percentile using shadowed regions. For fair comparison, we choose architectures' sizes such that the number of trainable parameters of ITTs is comparable to but upper-bounded by that of IOTs and explicit variants.}}
\label{fig:baseModel_appendix}
\end{figure*}

\subsection{Compare to Additional Baselines}
\label{subsec:appendix_baselines}
We provide comparison to previous works on different OpenAI Gym environments in Table~\ref{table:compare_with_salimans}. Note the ES baseline \citep{es} applies weights decay and performs fitness shaping, and they cap episode length at a constant steps for all workers, which they dynamically adjust during the training. Without applying any of these techniques used by the ES baseline to improve performance, and without fine-tuning parameters for the given reward thresholds (the number of ITT neural network layers and the value of $\sigma$ are the same as before), our ITT architecture reaches given reward thresholds with less timesteps. 

\begin{table}[h!]
  \begin{center}
    \begin{tabular}{l|c|c|c|r} 
      \textbf{Environment} & \textbf{Reward threshold} & \textbf{ITT timesteps} & \textbf{ES timesteps} & \textbf{TRPO timesteps}\\
      \hline
      Swimmer-v1& 128.25 & 1.07e+06 & 1.39e+06 &  4.59e+06 \\
      Hopper-v1&   877.45 & 0.69e+05 &  3.83e+05 & 7.29e+05\\
      Walker2d-v1&  957.68 & 3.16e+05 & 6.43e+05 &  1.55e+06 \\
    \end{tabular}
  \end{center}
\caption{\small{We present the number of timesteps needed to reach given reward thresholds set by previous works \citep{es,schulman2015trust}. The results were averaged over 6 random seeds. ``ES'' stands for the results of explicit policies provided by \citep{es}.}}
\label{table:compare_with_salimans}
\vspace{-2mm}
\end{table}

\subsection{ITT-RFT Results}
\label{subsec:appendix:itt-rft-results}

In Table~\ref{table:ablation_itt_rft}, we present results for ITT with random feature tree (ITT-RFT) described in Section~\ref{sec:rft}. Random feature mechanisms are shown to provide substantial computational gains with little sacrifice in performance\cite{choromanski2023efficient}. 

The results show that ITT-RFT achieves competitive performance in different environments, outperforming baselines. The result suggest that ITT-RFT could address the exponential sample complexity issue of implicit policies, without sacrificing performance. 

\begin{table}[h!]
  \begin{center}
    \begin{tabular}{l|c|c|c|c|r} 
      \textbf{Environment} & \textbf{Reward threshold} & \textbf{ITT} & \textbf{ITT-RFT} & \textbf{ES} & \textbf{TRPO}\\
      \hline
      Swimmer-v1& 128.25 & 1.07e+06 & 0.82e+06 & 1.39e+06 &  4.59e+06 \\
      Hopper-v1&   877.45 & 0.69e+05 & 0.67e+05 &  3.83e+05 & 7.29e+05\\
      Walker2d-v1&  957.68 & 3.16e+05 & 4.84e+05 & 6.43e+05 &  1.55e+06 \\
    \end{tabular}
  \end{center}
\caption{\small{We present the \textbf{number of timesteps} needed to reach given reward thresholds set by previous works \citep{es,schulman2015trust}. The results were averaged over 6 random seeds. ``ES'' stands for the results of explicit policies provided by \citep{es}. ``ITT-RFT timesteps'' stands for ITT with random feature tree, described in Section~\ref{sec:rft}}.}
\label{table:ablation_itt_rft}
\end{table}

\subsection{ITT-SRP Results on Additional Environments}

\label{subsec:appendix:additional_envs_itt_srp}
Figure~\ref{fig:hash_additional_exp} shows results of ITT-SRP on additional environments. For the additional environments, we use the number of action samples $N=2^{10}=1,024$ and the number of projection vectors $m=3$. ITT-SRPs maintain good performance even though the hash-space is very small (of size $2^{7}=128$).

\begin{figure*}[htb]
\centering
     \begin{tabular}{cc}
        \includegraphics[width=0.4\textwidth]{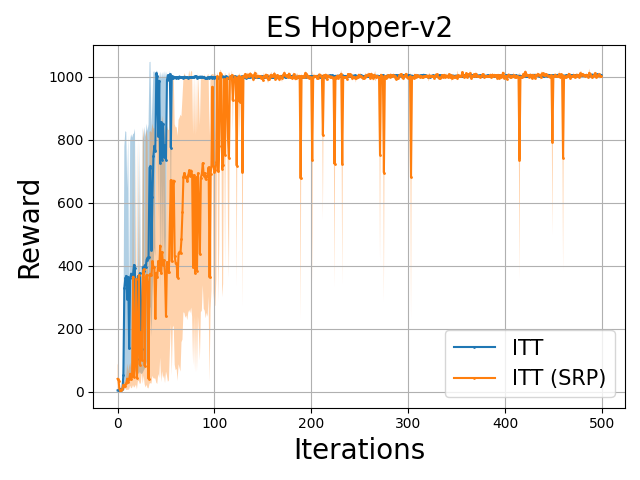} &
        \includegraphics[width=0.4\textwidth]{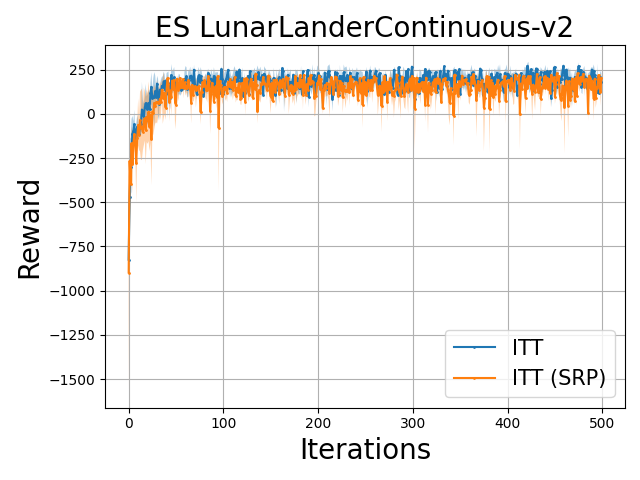} 
     \end{tabular}
    \caption{\small{Comparison of regular ITT and ITT with Signed Random Projection for fast MIP. We plot average curves obtained from 5 seeds, and we present the $80^{th}$ percentile and the $20^{th}$ percentile as shadowed regions.}}
    \label{fig:hash_additional_exp}
\end{figure*}

\subsection{ITT with Lazy Action Updates}
\label{subsec:appendix:lazy_wall_clock_time_savings}

We record wall-clock time of ITT with lazy action updates
on different OpenAI Gym environments. The results demonstrate that ITT with lazy action updates achieves substantial reduction in wall-clock time. 

Table~\ref{table:lazy_walltime} shows the wall-clock time (in minutes on 30 CPU cores on Google Cloud Compute) for regular ITT and ITT with lazy action updates (\textbf{18\%} training time reduction for HalfCheetah-v2, \textbf{9\%} for Walker2d-v2, \textbf{24\%} for Hopper-v2). 

Each experiment was run for 100 epochs, and each epoch has 1,000 timesteps. We require 1,000 action samples in each timestep. 

We conducted paired t-tests on wall-clock times between regular ITT and ITT (lazy). The null hypothesis is that there is no difference between the mean of the wall-clock time. The results show that p-values (for all tasks) are below 0.05, and thus we reject the null hypothesis. We conclude that the wall-clock time savings of ITT (lazy) is statistically significant at a 95\% confidence level.

\begin{table}[h!]
  \begin{center}
    \begin{tabular}{l|c|c|c} 
      \textbf{Environment} & \textbf{regular ITT} & \textbf{ITT (lazy)} & \textbf{Paired t-test p-value}\\
      \hline
      HalfCheetah-v2  & 25.22 & 20.67 & 3.31e-8\\ 
      Walker2d-v2 & 27.78 & 25.16 & 3.19e-2\\
      Hopper-v2 & 3.19 & 2.44 & 2.61e-3
      \end{tabular}
  \end{center}
\caption{\small{Comparison of the wall-clock time (in minutes) for regular ITT and ITT with lazy action updates. Each experiment was run for 100 epochs, and each epoch has 1,000 timesteps. \textbf{We also provide p-values from paired t-tests, showing ITT (lazy) wall-clock time savings are statistically significant at a 95\% confidence level (all p-values presented are below 0.05).}}}
\label{table:lazy_walltime}
\end{table}

\section{ABLATION STUDIES}
\label{sec:appendix_ablation}
\subsection{Performance Against the Number of Neurons}
We present ablation study results on ITT, using different network architectures. The results suggest that ITT consistently outperforms baselines, even when we significantly reduce the number of neurons in the network. 

In Table~\ref{table:ablation_num_neurons_timestep}, we present the number of timesteps to reach given reward thresholds set by previous works\citep{es,schulman2015trust}. In Table~\ref{table:ablation_num_params}, we provide the number of neurons in the network. In the tables, ITT-1 and ITT-2 stand for variants of ITT with the number of neurons significantly reduced. Notice that ITT-1 has less neurons than ITT, and ITT-2 has less neurons than ITT-1. We observe that the performance of ITT degrades slightly when the number of neurons in the network is decreased. 

\begin{table}[h!]
  \begin{center}
    \begin{tabular}{l|c|c|c|c|c|r} 
      \textbf{Environment} & \textbf{Reward threshold} & \textbf{ITT} & \textbf{ITT-1} & \textbf{ITT-2} & \textbf{ES} & \textbf{TRPO}\\
      \hline
      Swimmer-v1& 128.25 & 1.07e+06 & 1.14e+06 & 1.30e+06 & 1.39e+06 &  4.59e+06 \\
      Hopper-v1&   877.45 & 0.69e+05 & 1.34e+05 & 1.66e+05 &  3.83e+05 & 7.29e+05\\
      Walker2d-v1&  957.68 & 3.16e+05 & 3.30e+05 & 3.35e+05 & 6.43e+05 &  1.55e+06 \\
    \end{tabular}
  \end{center}
\caption{\small{We present the \textbf{number of timesteps} needed to reach given reward thresholds set by previous works \citep{es,schulman2015trust}. The results were averaged over 6 random seeds. ``ES'' stands for the results of explicit policies provided by \citep{es}. ``ITT-1'' and ``ITT-2'' stand for variants of ITT with the number of neurons significantly reduced.}}
\label{table:ablation_num_neurons_timestep}
\end{table}

\begin{table}[h!]
  \begin{center}
    \begin{tabular}{l|c|c|r} 
      \textbf{Environment} & \textbf{ITT} & \textbf{ITT-1} & \textbf{ITT-2}\\
      \hline
      Swimmer-v1& 20 & 18 & 14\\
      Hopper-v1& 42 & 32 & 28\\
      Walker2d-v1& 246 & 210 & 175\\
    \end{tabular}
  \end{center}
\caption{The dimensionality of the learnable $\theta \in \mathbb{R}^{D}$ for different environments.}
\label{table:ablation_num_params}
\end{table}

\subsection{ITT-SRP Runtime Against the Sampling Amount}
We record wall-clock time of ITT-SRP on different OpenAI Gym environments. The results demonstrate that ITT-SRP achieves substantial reduction in wall-clock time. 

While running IOT, to ensure sufficient covarege of action space, we need to sample a large number of actions at each timestep (usually the number of samples is exponential in the dimension of the action space). To showcase the effectiveness of ITT-SRP in resolving the sample complexity of IOT, we require $N=2^c$ action samples per timestep. 

Table~\ref{table:hash_walltime} shows the wall-clock time (in minutes on 30 CPU cores on Google Cloud Compute) for ITT-SRP and IOT. We conclude that when the number of samples required gets large, ITT-SRP trains much faster than IOTs (for HalfCheetah-v1 \textbf{92\%} training time reduction for $c=14$, \textbf{89\%} for $c=13$, \textbf{81\%} for $c=12$). 

Furthermore, IOT wall-clock time grows exponentially in the parameter $c$, and ITT-SRP wall-clock time only grows linearly in the parameter $c$.

We also conducted paired t-tests on wall-clock times between ITT-SRP and IOT. The null hypothesis is that there is no difference between the mean of the wall-clock time. The results show that p-values (for all tasks) are below 0.05, and thus we reject the null hypothesis. We conclude that the wall-clock time savings of ITT-SRP is statistically significant at a 95\% confidence level.

\begin{table}[h!]
  \begin{center}
    \begin{tabular}{l|c|c|c|c|c} 
      \textbf{Environment} & \textbf{Choice of $c$} & \textbf{\# samples / timestep} & \textbf{ITT-SRP} & \textbf{IOT} &\textbf{p-value}\\
      \hline
      HalfCheetah-v1 & 10 & 1024 & 15.15 & 30.71 & 2.62e-19\\ 
      HalfCheetah-v1 & 11 & 2028 & 17.68 & 57.50 & 1.81e-20\\ 
      HalfCheetah-v1 & 12 & 4096 & 20.02 & 104.51 & 8.37e-20\\ 
      HalfCheetah-v1 & 13 & 8192 & 24.17 & 218.33 & 4.19e-21\\ 
      HalfCheetah-v1 & 14 & 16384 & 31.67 & 421.67 & 8.72e-22\\ 
      Walker2d-v1& 10 & 1024 & 19.71 & 39.73 & 7.24e-15\\
      Walker2d-v1& 11 & 2048 & 20.42 & 62.52 & 6.18e-16\\
      Walker2d-v1& 12 & 4096 & 22.02 & 106.67 & 4.72e-17\\
      Walker2d-v1& 13 & 8192 & 24.77 & 203.34 & 4.34e-21\\
      Walker2d-v1& 14 & 16384 & 29.82 & 386.67 & 1.99e-22\\
    \end{tabular}
  \end{center}
\caption{\small{Comparison of the wall-clock time (in minutes) for ITT-SRP and IOT. \textbf{We require $N = 2^c$ actions samples at each timestep.} Each experiment was run for 100 epochs, and each epoch has 1,000 timesteps. \textbf{We also provide p-values from paired t-tests, showing ITT-SRP wall-clock time savings are statistically significant at a 95\% confidence level (all p-values presented are below 0.05).}}}
\label{table:hash_walltime}
\end{table}


\subsection{Use Regular ITT in Training and ITT-SRP in Inference}
ITT-SRPs can be applied in both training and/or inference. For higher accuracy, one can train with the regular ITT and run inference with ITT-SRP. 

In Table~\ref{table:appendix_itt_train_srp_inference}, we present the results for using regular ITT in training and using ITT-SRP in inference. The results show that doing so achieves slightly higher reward than using ITT-SRP in both training and inference.

\begin{table}[h!]
  \begin{center}
    \begin{tabular}{l|c|c|c|c} 
      \textbf{Training} & \textbf{ITT} & \textbf{ITT} & \textbf{ITT-SRP} & \\
      \textbf{Inference} & \textbf{ITT} & \textbf{ITT-SRP} & \textbf{ITT-SRP} & \textbf{\# iter}\\
      \hline
      Swimmer-v2& $\textbf{344.13} \pm 5.18$ & $\underline{341.38} \pm 7.93$ & $333.84 \pm 7.26$ & 500\\
      LunarLanderC-v2& $\textbf{157.38} \pm 71.81$ & $\underline{154.41} \pm 71.01$ & $149.15 \pm 69.27$ & 500 \\
      Hopper-v2& $\textbf{1007.95} \pm 3.66$ & $999.12 \pm 11.39$ & $\underline{1000.54} \pm 7.74$ & 500\\
      HalfCheetah-v2 & $\textbf{2866.02} \pm 416.81$ & $ \underline{2690.64} \pm 139.05$ & $2565.74 \pm 143.15$ & 4000
    \end{tabular}
  \end{center}
\caption{\small{We present final average scores} over $s=5$ random seeds together with their std. The best architecture is in bold font and the second best is underscored. The ``LunarLanderC-v2' stands for the LunarLanderContinuous-v2 environment.}
\label{table:appendix_itt_train_srp_inference}
\end{table}


\section{Paired T-test Results}
\label{sec:appendix_t_test}
To demonstrate that ITT achieves higher scores than IOT and explicit policies, we provide paired t-test results as evidence of statistical significance. 

In Table~\ref{table:paired-t-test-final-average-return}, we present p-values from paired t-tests of final scores of different methods. 

ITT vs IOT: ITT achieves significantly higher scores than IOT on all tasks except Walker2d-v2, where ITT is the second best among the three architectures. The p-values (ITT \& IOT paired t-test) are below 0.05 for all tasks except Half-Cheetah-v2. Thus, the null hypothesis (no difference between the means of IOT and ITT final score) is rejected given significance level 0.05. We conclude that there is \underline{statistically significant difference between the means of final returns of ITT and IOT}. 

ITT vs Explicit: ITT achieves significantly higher scores than explicit policies on more difficult tasks (Hopper-v2, HalfCheetah-v2, Walker2d-v2). The p-values (ITT \& IOT paired t-test) are below 0.05 for Hopper-v2 and HalfCheetah-v2. Thus, the null hypothesis (no difference between the means of IOT and ITT final score) is rejected given significance level 0.05. We conclude that there is \underline{statistically significant difference between the means of final returns of ITT and explicit policies} on Hopper-v2 and HalfCheetah-v2, which are the more difficult tasks.  

For simpler tasks, we also look at the number of iterations needed to achieve given reward thresholds. The reward thresholds are set at the 90\% of final average return achieved by  IOT. 

In Table~\ref{table:paired-t-test-n-iter}, we present p-values from paired t-tests of the number of iterations needed. 

ITT vs IOT: for most environments, the p-values are below 0.05. Thus, the null hypothesis (no difference between the means of IOT and ITT number of iterations to reach given reward threshold) is rejected given significance level 0.05. We conclude that there is statistically significant difference between the number of iterations needed to reach given reward thresholds. 

ITT vs Explicit: we do not include explicit in this comparison, because their performance are close on simpler tasks. For more difficult tasks (Hopper-v2, HalfCheetah-v2, Walker2d-v2), Explicit cannot even reach 80\% of final average return of ITT.

\begin{table}[h!]
  \begin{center}
    \label{tab:scores}
    \begin{tabular}{l|c|c|c|c|c} 
       & \textbf{ITT \& IOT} & \textbf{ITT \& Explicit} &\textbf{ITT} &{IOT} & \textbf{Explicit}\\
       \textbf{Environment} & \textbf{p-value} & \textbf{p-value} &\textbf{return} &{return} & \textbf{Return}\\      
      \hline
      Swimmer-v2& \underline{7.09e-6} & 2.65e-1 & 344.13 & 75.54 & \textbf{347.67} \\
      LunarLanderContinuous-v2& \underline{9.73e-6} & 1.90e-1 & \textbf{157.38} & -72.72 & 62.85\\
      Hopper-v2 & \underline{8.05e-9} & \underline{5.68e-9} & \textbf{2670.37} & 1036.52 & 1060.89\\
      HalfCheetah-v2 & 1.69e-1 & \underline{2.48e-4} & \textbf{2866.02} & 2696.29 & 1845.64\\
      Walker2d-v2& \underline{2.89e-3} & 1.63e-1 & 1897.87 & \textbf{2909.85} & 1346.93\\
      MountainCar-v0 & \underline{4.28e-05} & 5.27e-1 & \textbf{-113.50} & -200 & -143.40\\
      MountainCarContinuous-v0 & \underline{4.08e-3} & \underline{4.38e-2} & \textbf{89.17} & 25.96 & 52.34\\
      InvertedPendulumBulletEnv-v0 & \underline{2.86e-18} & N/A & \textbf{1000.00} & 27.10 & \textbf{1000.00}\\
    \end{tabular}
  \end{center}
\caption{We present p-values from paired t-tests of returns of different methods. We underline results that are statistically significant, at a 95\% confidence level. The p-values are presented in three significant figures. We also provide final average returns. The best architecture is in bold font.}
\label{table:paired-t-test-final-average-return}
\end{table}

\begin{table}[h!]
  \begin{center}
    \label{tab:scores}
    \begin{tabular}{l|c|c|c|c} 
       & \textbf{ITT \& IOT} & \textbf{ITT} & \textbf{IOT} & \textbf{Reward}\\
      \textbf{Environment} & \textbf{p-value} & \textbf{\# iter} & \textbf{\# iter} & \textbf{Threshold}\\
      \hline
      Swimmer-v2& \underline{2.96e-05} & 10.80 &  420.10 & 67.98\\
      LunarLanderContinuous-v2& \underline{2.02e-07} & 12.20 & 386.80 & -65.44\\
      MountainCar-v0 & \underline{1.52e-07} & 733.80 & 2000.00 & -180.0 \\
      MountainCarContinuous-v0 & \underline{6.38e-2} & 75.80 & 125.00 & 23.36\\
      InvertedPendulumBulletEnv-v0 & \underline{1.12e-07} & 9.50 & 112.10 & 24.39\\
    \end{tabular}
  \end{center}
\caption{We present p-values from paired t-tests of the number of iterations used by each method to achieve given reward threshold. We underline results that are statistically significant, at a 95\% confidence level. The p-values are presented in three significant figures. The reward thresholds are set at the 90\% of final average return achieved by  IOT. }
\label{table:paired-t-test-n-iter}
\end{table}

\section{Confidence Interval Plots}
In the main paper (Figure~\ref{fig:baseModel}), we presented average score and quantiles, which are key metrics commonly used \cite{rlintro,asebo,elmachtoub2023estimate,elmachtoub2023balanced}. To further showcase the strength of proposed ITT architecture, below (Table~\ref{fig:baseModel_confidence_intervals}) we present the scores in a similar fashion as in Figure~\ref{fig:baseModel}, but using 95\% confidence intervals instead of 90\% and 10\% percentiles.

\begin{figure*}[!htb]
\minipage{0.32\textwidth}
  \includegraphics[width=\linewidth]{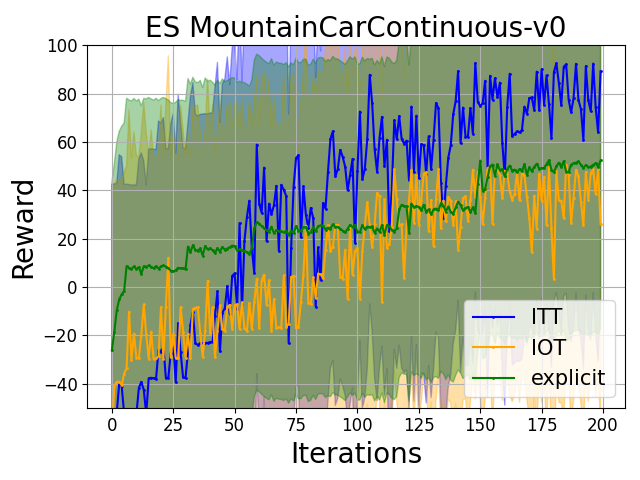}
\endminipage\hfill
\minipage{0.32\textwidth}
  \includegraphics[width=\linewidth]{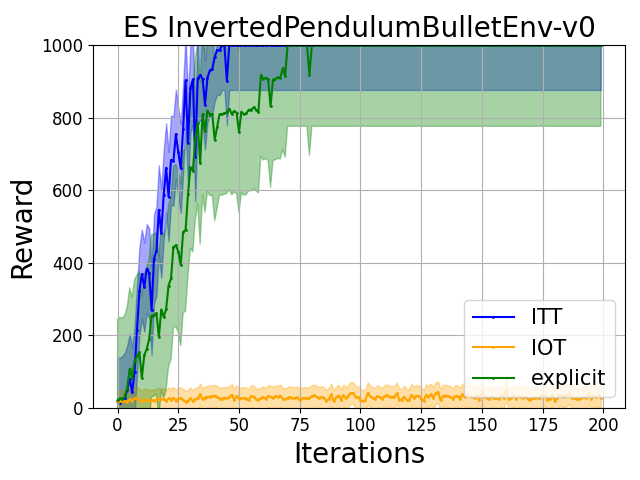}
\endminipage\hfill
\minipage{0.32\textwidth}%
  \includegraphics[width=\linewidth]{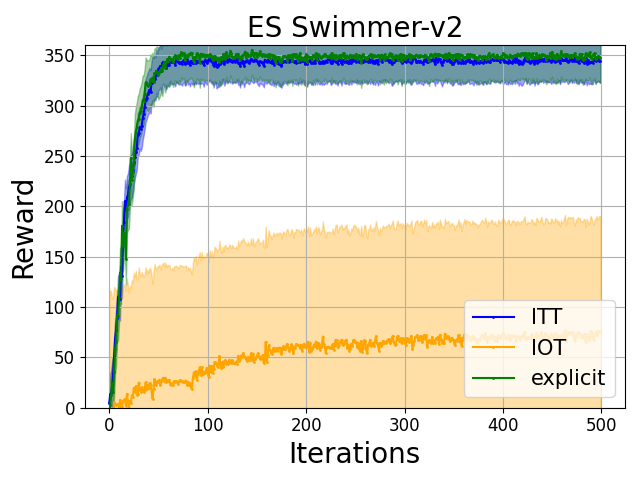}
\endminipage
\end{figure*}

\begin{figure*}[!htb]
\minipage{0.32\textwidth}
  \includegraphics[width=\linewidth]{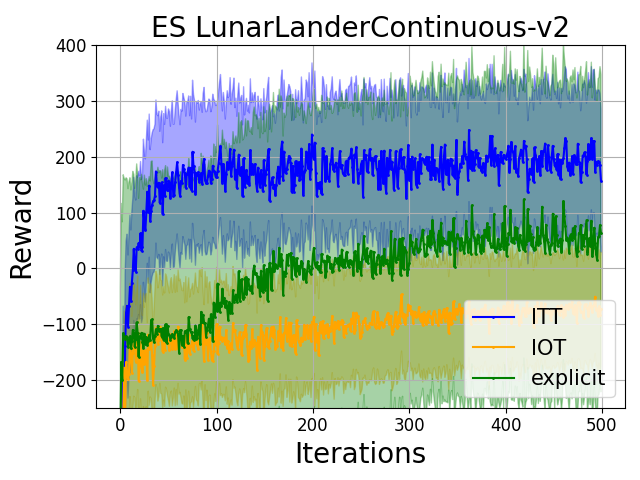}
\endminipage\hfill
\minipage{0.32\textwidth}
  \includegraphics[width=\linewidth]{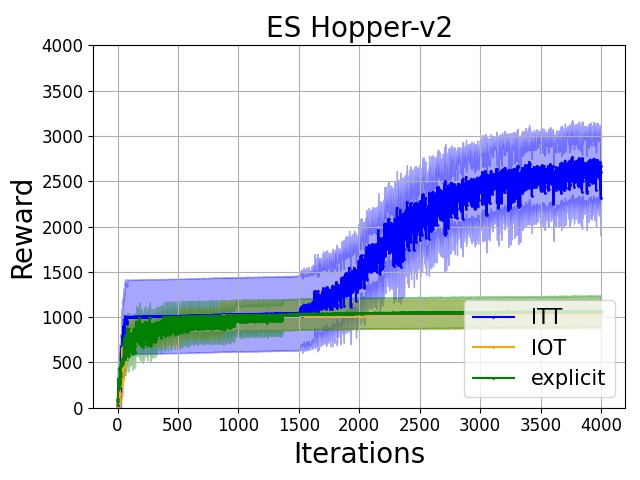}
\endminipage\hfill
\minipage{0.32\textwidth}%
  \includegraphics[width=\linewidth]{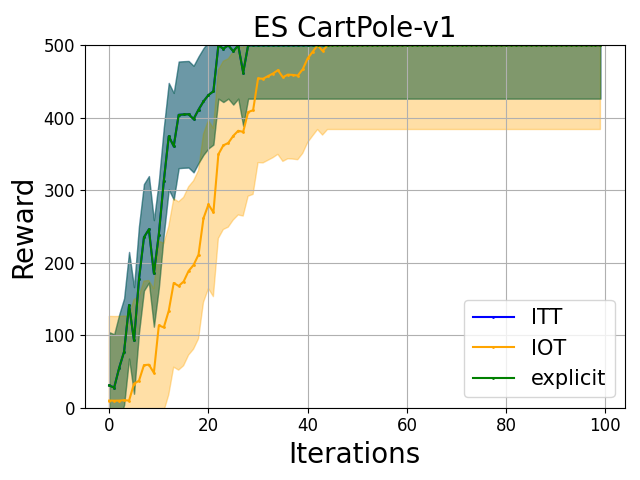}
\endminipage
\end{figure*}

\begin{figure*}[!htb]
\minipage{0.32\textwidth}
  \includegraphics[width=\linewidth]{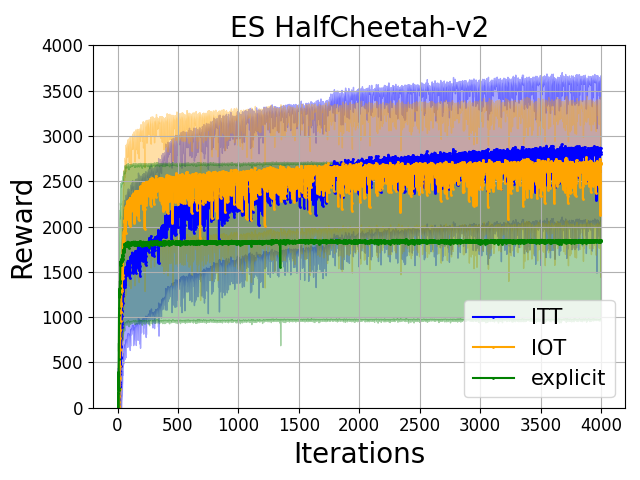}
\endminipage\hfill
\minipage{0.32\textwidth}
  \includegraphics[width=\linewidth]{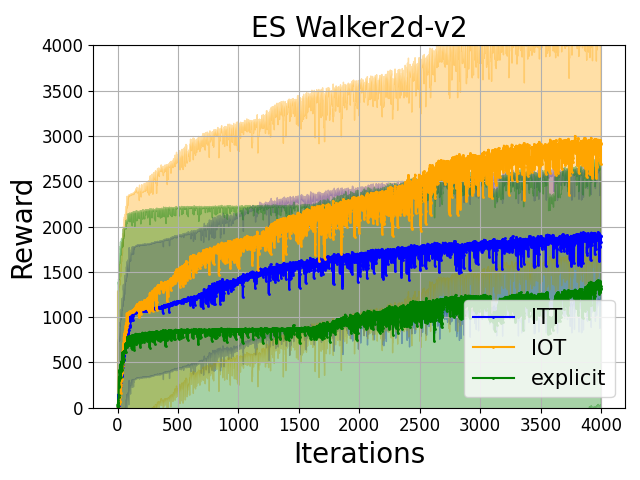}
\endminipage\hfill
\minipage{0.32\textwidth}%
  \includegraphics[width=\linewidth]{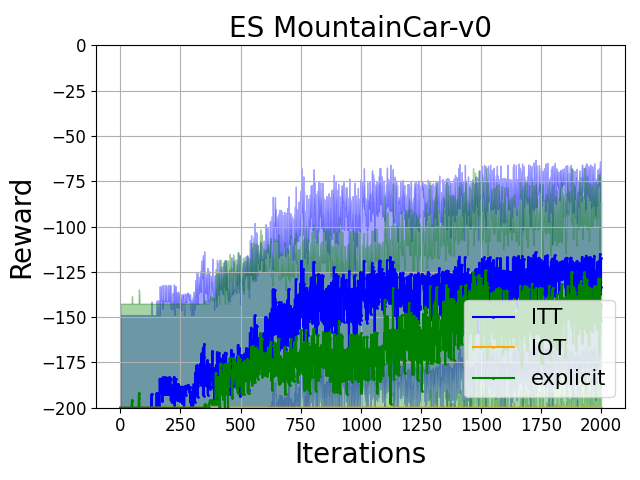}
\endminipage
\caption{\small{The comparison of the performance of different policy-architectures: ITTs, IOTs and explicit on various $\mathrm{OpenAI}$ $\mathrm{Gym}$ tasks. We plot average curves and present the 95\% confidence intervals using shadowed regions.}}
\label{fig:baseModel_confidence_intervals}
\end{figure*}
\section{Theoretical Results}
\label{sec:theoretical_results}

\subsection{Proof of results in in Section~\ref{sec:itt}}

\begin{proof}[Proof of Lemma~\ref{lem:approxSoftmaxSampling}]
\begin{align}
&\mathbb{P}[\widehat{\pi_{\theta}}(\mathbf{s})\in \mathcal{A}^{*}_{v_{l}} | \widehat{\pi_{\theta}}(\mathbf{s}) \in \mathcal{A}^{*}_{v_{l}} \cup \mathcal{A}^{*}_{v_{r}}] 
\nonumber\\ 
=&\frac{\sum_{\mathbf{a} \in \mathcal{A}^{*}_{v_{l}}}\exp(-E_{\theta}(\mathbf{s},\mathbf{a}))}{\sum_{\mathbf{a} \in \mathcal{A}^{*}_{v_{l}} \cup \mathcal{A}^{*}_{v_{r}}}\exp(-E(\mathbf{s},\mathbf{a}))}  \nonumber\\
=&\frac{\sum_{\mathbf{a} \in \mathcal{A}^{*}_{v_{l}}}\exp\left\{\mathrm{K}\left(l_{\mathcal{S}}^{\theta_{1}}(\mathbf{s}),l_{\mathcal{A}}^{\theta_{2}}(\mathbf{a})\right)\right\}}{\sum_{\mathbf{a} \in \mathcal{A}^{*}_{v_{l}} \cup \mathcal{A}^{*}_{v_{r}}}\exp\left\{\mathrm{K}\left(l_{\mathcal{S}}^{\theta_{1}}(\mathbf{s}),l_{\mathcal{A}}^{\theta_{2}}(\mathbf{a})\right)\right\}} \nonumber\\
\approx & 
\frac{\sum_{\mathbf{a} \in \mathcal{A}^{*}_{v_{l}}}\exp\left(\phi(l_{\mathcal{S}}^{\theta_{1}}(\mathbf{s}))^{\top}\phi(l_{\mathcal{A}}^{\theta_{2}}(\mathbf{a}))\right)}{\sum_{\mathbf{a} \in \mathcal{A}^{*}_{v_{l}} \cup \mathcal{A}^{*}_{v_{r}}}\exp\left(\phi(l_{\mathcal{S}}^{\theta_{1}}(\mathbf{s}))^{\top}\phi(l_{\mathcal{A}}^{\theta_{2}}(\mathbf{a}))\right)} \nonumber \\
\approx &
\frac{\psi(\phi(l_{\mathcal{S}}^{\theta_{1}}(\mathbf{s})))^{\top}\sum_{\mathbf{a} \in \mathcal{A}^{*}_{v_{l}}}\psi(\phi(l_{\mathcal{A}}^{\theta_{2}}(\mathbf{a})))}{\psi(\phi(l_{\mathcal{S}}^{\theta_{1}}(\mathbf{s})))^{\top}\sum_{\mathbf{a} \in \mathcal{A}^{*}_{v_{l}} \cup \mathcal{A}^{*}_{v_{r}}}\psi(\phi(l_{\mathcal{A}}^{\theta_{2}}(\mathbf{a})))}\nonumber\\
=&\frac{\psi(\phi(l_{\mathcal{S}}^{\theta_{1}}(\mathbf{s})))^{\top}\xi(v_{l})}{\psi(\phi(l_{\mathcal{S}}^{\theta_{1}}(\mathbf{s})))^{\top}\left(\xi(v_{l})+\xi(v_{r})\right)},
\label{eq:approxsoftmaxsampling}
\end{align}

\end{proof}

\subsection{Proof of results in Section~\ref{sec:es}}

We start with introducing notations that help us simplify the proofs. 

\begin{definition}[AT and FD ES-gradient estimator]
The antithetic ES-gradient estimator and the forward finite difference ES-gradient estimator applying orthogonal samples are defined as
\begin{align}
\label{eq:orthogonal_at_fd_gradient_detailed}
\widehat{\nabla}_{M}^{\mathrm{AT}, \text { ort }} F_{\sigma}(\theta)
&\coloneqq
\frac{1}{2 \sigma M} \sum_{i=1}^{M} F^{AT(i)}
\quad\text{where}\quad
F^{AT(i)} \coloneqq F\left(\theta+\sigma \ve_{i}\right) \ve_{i}-F\left(\theta-\sigma \ve_{i}\right) \ve_{i}\\
\widehat{\nabla}_{M}^{\mathrm{FD}, \text { ort }} F_{\sigma}(\theta)
&\coloneqq
\frac{1}{\sigma M} \sum_{i=1}^{M} F^{FD(i)}
\quad \text{where} \quad
F^{FD(i)} \coloneqq F\left(\theta+\sigma \ve_{i}\right) \ve_{i}-F\left(\theta\right) \ve_{i},
\end{align}
where $\left(\ve_{i}\right)_{i=1}^{M}$ have marginal distribution $\mathcal{N}(\mathbf{0}, I_{D})$, and $\left(\ve_{i}\right)_{i=1}^{M}$ are conditioned to be pairwise-orthogonal. 
\end{definition}

\begin{definition}[Gaussian smoothing]
\label{def:gaussian_smoothing}
The Gaussian smoothing of $F(x)$ is defined as
\begin{align}
F_\sigma (\theta)=\frac{1}{\kappa}\int F(\theta+\sigma\ve)e^{-\frac{1}{2}||\ve||_2^2} d\ve\quad\text{, where }\kappa=(2\pi)^{d/2},
\end{align}
and its gradient is 
\begin{align}
\gradGaus=\frac{1}{\sigma\kappa}\int F(\theta+\sigma\ve)e^{-\frac{1}{2}||\ve||_2^2}\ve d\ve .
\end{align}

\end{definition}



\begin{assumption}
\label{assp:quadratic_objective}
Assume $F(\cdot)$ is quadratic. Under this assumption, the gradient $\grad$ and the Hessian $\hess$ exist for any $\theta\in\mathbb{R}^D$, and 
$$F(\theta+\sigma\epsilon)=F(\theta)+\sigma\grad^\top \epsilon+\frac{\sigma^2}{2}\ve^\top \hess\ve.$$
\end{assumption}

 
\begin{theorem}
\label{thm:mse_at_fd}
Suppose Assumption~\ref{assp:quadratic_objective} holds. The mean squared error of the AT ES-gradient estimator applying orthogonal samples is
\begin{align*}
\operatorname{MSE}\left(\hat{\nabla}_{M}^{\mathrm{AT}, \mathrm{ort}} F_{\sigma}(\theta)\right) 
&\coloneqq\mathbb{E}\left[\left\|\hat{\nabla}_{N}^{\mathrm{AT}, \mathrm{ort}} F_{\sigma}(\theta)-\nabla F_{\sigma}(\theta)\right\|_{2}^{2}\right]\\
&=\frac{1}{M} \mathbb{E}\left[\left\| (\grad^\top \ve)\ve \right\|_{2}^{2}\right]-\left\|\nabla F_{\sigma}(\theta)\right\|_{2}^{2}.
\end{align*}
The mean squared error of the FD ES-gradient estimator  applying orthogonal samples is 
\begin{align*}
\operatorname{MSE}\left(\hat{\nabla}_{N}^{\mathrm{FD}, \mathrm{ort}} F_{\sigma}(\theta)\right) 
&\coloneqq\mathbb{E}\left[\left\|\hat{\nabla}_{N}^{\mathrm{FD}, \mathrm{ort}} F_{\sigma}(\theta)-\nabla F_{\sigma}(\theta)\right\|_{2}^{2}\right]\\
&=\frac{1}{M} \mathbb{E}\left[\left\| (\grad^\top \ve+\frac{\sigma^2}{2}\ve^\top \hess \ve)\ve \right\|_{2}^{2}\right] -\left\|\nabla F_{\sigma}(\theta)\right\|_{2}^{2}.
\end{align*}
\end{theorem}

\begin{remark}
Theorem~\ref{thm:mse_at_fd} can be extended to \textbf{general functions} $F(\cdot)$. Classical results on Gaussian smoothing gradient estimators, which motivate the use of ES-gradient estimators, require the objective to be twice continuously differentiable \cite{Nesterov2017RandomGM}. Their error bound results rely on second order Taylor expansion. Our results for quadratic objectives can be generalized to non-quadratic functions, by imposing twice continuously differentiable assumptions as in \cite{Nesterov2017RandomGM} and taking a second order Taylor expansion of a general function, after which we bound the residual terms using the smoothness assumption. One may also take higher order Taylor polynomials of $F(\cdot)$ and bound the residual terms under suitable regularity conditions. 
\end{remark}

Since we have orthogonal samples, we have the following Lemma. 

\begin{lemma}
\label{lem:mse_at_fd}
Assume $M\leq D$, we have
\begin{align*}
\operatorname{MSE}\left(\hat{\nabla}_{M}^{\mathrm{AT}, \mathrm{ort}} F_{\sigma}(\theta)\right) =&\frac{D+2}{M}||\grad||_2^2 - \left\|\nabla F_{\sigma}(\theta)\right\|_{2}^{2}\\
\operatorname{MSE}\left(\hat{\nabla}_{M}^{\mathrm{FD}, \mathrm{ort}} F_{\sigma}(\theta)\right) =&\frac{D+2}{M}||\grad||_2^2+\frac{(D+4)\sigma^4}{4M}||\hess||_F^2\\
& +\frac{(D+2)\sigma^4}{M}\left(\sum_{i=1}^D  \hess_{ii}^2\right)- \left\|\nabla F_{\sigma}(\theta)\right\|_{2}^{2},
\end{align*}
and therefor
\begin{align*}
&\operatorname{MSE}\left(\hat{\nabla}_{M}^{\mathrm{FD}, \mathrm{ort}} F_{\sigma}(\theta)\right) -
\operatorname{MSE}\left(\hat{\nabla}_{M}^{\mathrm{AT}, \mathrm{ort}} F_{\sigma}(\theta)\right)\\
=&\frac{(D+4)\sigma^4}{4M}||\hess||_F^2+\frac{(D+2)\sigma^4}{M}\left(\sum_{i=1}^D  \hess_{ii}^2\right).
\end{align*}
\end{lemma}

The proofs of Theorem~\ref{thm:mse_at_fd} and Lemma~\ref{lem:mse_at_fd} are at the end of Section~\ref{sec:theoretical_results}. 

\begin{remark}
Suppose Assumption~\ref{assp:quadratic_objective} holds. We observe that evaluating $\hat{\nabla}_{M}^{\mathrm{AT}, \mathrm{ort}} F_{\sigma}(\theta)$ requires 2M queries of $F(\cdot)$ and evaluating $\hat{\nabla}_{M}^{\mathrm{FD}, \mathrm{ort}} F_{\sigma}(\theta)$ requires only M+1 queries of $F(\cdot)$. Consequently, FD-gradient estimator is
preferred when $\sigma^4 ||\hess||_F^2 \ll  ||\grad||_2^2$, and AT-gradient estimator is preferred when $\sigma^4 ||\hess||_F^2 \gg  ||\grad||_2^2$, which is highly likely when $\sigma$ is large.
\end{remark}

\begin{proof}[proof of Theorem~\ref{thm:mse_at_fd}]
{\bf AT ES-gradient estimator.}
\begin{align*}
&\operatorname{MSE}\left(\hat{\nabla}_{N}^{\mathrm{AT}, \mathrm{ort}} F_{\sigma}(\theta)\right) =\mathbb{E}\left[\left\|\frac{1}{M} \sum_{i=1}^M F^{AT(i)}-\nabla F_{\sigma}(\theta)\right\|_{2}^{2}\right]\\
=&\mathbb{E}\left[\left\|\frac{1}{M} \sum_{i=1}^M F^{AT(i)}\right\|_{2}^{2}\right]-\left\|\nabla F_{\sigma}(\theta)\right\|_{2}^{2}
\end{align*}
The first term is 
\begin{align*}
&\mathbb{E}\left[\left\|\frac{1}{M} \sum_{i=1}^M F^{AT(i)}\right\|_{2}^{2}\right]
=\frac{1}{M^2}\left(\sum_{i=1}^M \mathbb{E}\left[\left\|F^{AT(i)}\right\|_{2}^{2}\right]+\sum_{i \neq j} \mathbb{E}\left[\left\langle F^{AT(i)}, F^{AT(j)}\right\rangle\right]\right)\\
=&\frac{1}{M^2}\left(\sum_{i=1}^M \mathbb{E}\left[\left\|F^{AT(i)}\right\|_{2}^{2}\right]\right)
=\frac{1}{M^2}\left(\sum_{i=1}^M \mathbb{E}\left[\left\| \frac{1}{2\sigma}(F(\theta+\sigma\ve_i)\ve_i-F(\theta-\sigma\ve_i)\ve_i)\right\|_{2}^{2}\right]\right)\\
=&\frac{1}{M} \mathbb{E}\left[\left\| \frac{1}{2\sigma}(F(\theta+\sigma\epsilon)\epsilon-F(\theta-\sigma\epsilon)\epsilon)\right\|_{2}^{2}\right]
=\frac{1}{M} \mathbb{E}\left[\left\| (\grad^\top \ve)\ve \right\|_{2}^{2}\right], 
\end{align*}
where the second equality is by orthogonality of $\ve_i$; the fourth equality is because $\ve_i are$ i.i.d.; the last equality is by Assumption~\ref{assp:quadratic_objective}. 

{\bf FD ES-gradient estimator.} For simplicity of presentation, we abbreviate the first few steps, which are the same as that of the AT ES-gradient estimator. 

\begin{align*}
&\frac{1}{M^2}\left(\sum_{i=1}^M \mathbb{E}\left[\left\|F^{FD(i)}\right\|_{2}^{2}\right]\right)
=\frac{1}{M^2}\left(\sum_{i=1}^M \mathbb{E}\left[\left\| \frac{1}{\sigma}(F(\theta+\sigma\ve_i)\ve_i-F(\theta)\ve_i)\right\|_{2}^{2}\right]\right)\\
=&\frac{1}{M} \mathbb{E}\left[\left\| \frac{1}{\sigma}(F(\theta+\sigma\epsilon)\epsilon-F(\theta)\epsilon)\right\|_{2}^{2}\right]
=\frac{1}{M} \mathbb{E}\left[\left\| \ve (\grad^\top \ve+\frac{\sigma^2}{2}\ve^\top \hess \ve) \right\|_{2}^{2}\right], 
\end{align*}
where the second equality is because $\ve_i$ are i.i.d., and the third equality is by Assumption~\ref{assp:quadratic_objective}.
\end{proof}

\begin{proof}[Proof of Lemma~\ref{lem:mse_at_fd}]
{\bf AT ES-gradient estimator.}
\begin{align*}
&\mathbb{E}\left[\left\| (\grad^\top \ve)\ve \right\|_{2}^{2}\right]=
\sum_{i,j,k}\grad_i\grad_j \mathbb{E}\left[\ve_i\ve_j\ve_k^2 \right]
=\sum_{i,k}\grad_i^2\mathbb{E}\left[\ve_i^2\ve_k^2 \right]\\
=&\sum_{i=1}^D\grad_i^2\mathbb{E}\left[\ve_i^4 \right]+\sum_{i=1}^D \grad_i^2\sum_{k\neq i}\mathbb{E}\left[\ve_i^2\ve_k^2 \right]=(D+2)||\grad||_2^2,
\end{align*}
where the second equality is because odd moments of Gaussian r.v.s are zero. By Theorem~\ref{thm:mse_at_fd}, we have 
\begin{align*}
&\operatorname{MSE}\left(\hat{\nabla}_{M}^{\mathrm{AT}, \mathrm{ort}} F_{\sigma}(\theta)\right) 
\coloneqq\mathbb{E}\left[\left\|\hat{\nabla}_{N}^{\mathrm{AT}, \mathrm{ort}} F_{\sigma}(\theta)-\nabla F_{\sigma}(\theta)\right\|_{2}^{2}\right]\\
=&\frac{1}{M} \mathbb{E}\left[\left\| (\grad^\top \ve)\ve \right\|_{2}^{2}\right]-\left\|\nabla F_{\sigma}(\theta)\right\|_{2}^{2}
=\frac{D+2}{M}||\grad||_2^2 - \left\|\nabla F_{\sigma}(\theta)\right\|_{2}^{2}.
\end{align*}

{\bf FD ES-gradient estimator.}

\begin{align*}
&\mathbb{E}\left[\left\| \ve (\grad^\top \ve+\frac{\sigma^2}{2}\ve^\top \hess \ve) \right\|_{2}^{2}\right]\\
=&\mathbb{E}\left[\left\| \left(\sum_{i=1}^D \grad_i \ve_i+\frac{\sigma^2}{2}\sum_{i=1}^D\sum_{j=1}^D \ve_i\hess_{ij}\ve_j\right)\ve\right\|_2^2 \right]\\
=&\mathbb{E}\left[
\sum_{k=1}^D \ve_k^2\left(\sum_{i=1}^D \grad_i\ve_i+\frac{\sigma^2}{2}\sum_{i=1}^D\sum_{j=1}^D \hess_{ij}\ve_j\right)^2\right]\\
=&\sum_{i,j,k}\grad_i\grad_j \mathbb{E}\left[\ve_i\ve_j\ve_k^2 \right]
+\sum_{i,j,k,l}\mathbb{E}\left[
\ve_k^2\grad_i\ve_i\sigma^2\ve_j\ve_l\hess_{jl}
\right]\\
&+\sum_{i,j,k}\mathbb{E}\left[
\ve_k^2\frac{\sigma^4}{4}\ve_i^2\ve_j^2\hess_{ij}^2
\right],
\end{align*}
The second term above equals zero, because the odd moments of Gaussian random variables are zero (in a degree 5 polynomial of Gaussian r.v.s, one term must be raised to an odd power); the first term equals $(D+2)||\grad||_2^2$
by the same argument as for the AT ES-gradient estimator; the third term is 
\begin{align*}
&\sum_{i,j,k}\mathbb{E}\left[
\ve_k^2\frac{\sigma^4}{4}\ve_i^2\ve_j^2\hess_{ij}^2
\right]
=\frac{\sigma^4}{4}  \left(15\sum_i\hess_{ii}^2\right)
+\frac{\sigma^4}{4} 3(D-1)\sum_i \hess_{ii}^2\\
&+\frac{\sigma^4}{4} 3\sum_i \sum_{j\neq i}\hess_{ij}^2
+\frac{\sigma^4}{4} 3\sum_j \sum_{i\neq j}\hess_{ij}^2
+\frac{\sigma^4}{4}  (D-2)\left(\sum_i\sum_{j\neq i}\hess_{ij}^2 \right)\\
=&\frac{(D+4)\sigma^4}{4}\|\hess\|_F^2+\frac{(8+4D)\sigma^4}{4}\sum_{i=1}^D \hess_{ii}^2,
\end{align*}
where in the first equality, the five terms correspond to $i=j=k,i=j\neq k,i=k\neq j,j=k\neq i$ and distinct $i,j,k$ respectively.

\end{proof}

\end{document}